\documentclass[utf8]{frontiersinFPHY_FAMS} 

\setcitestyle{square} 
\usepackage[onehalfspacing]{setspace}

\usepackage{amsmath}
\usepackage{amsthm, amsfonts, amssymb}
\usepackage{mathrsfs}
\usepackage{enumerate}

\usepackage{algpseudocode}
\usepackage{algorithm}
\def\algbackskip{\hskip\dimexpr-\algorithmicindent+\labelsep}
\def\LState{\State \algbackskip}

\usepackage{url,hyperref,lineno,microtype,subcaption}

\DeclareMathOperator*{\argmin}{arg\,min}    
\DeclareMathOperator*{\esssup}{ess\,sup}    

\usepackage{xcolor}
\definecolor{DarkGreen}{rgb}{0.1,0.5,0.1}

\newtheorem{theorem}{Theorem}
\newtheorem{remark}{Remark}
\newtheorem{lemma}{Lemma}
\newtheorem{definition}{Definition}

\newcommand{\OPtext}[1]{\textcolor{black}{#1}}

\def\hide #1 {}
\long\def\longhide #1 {}

\def\keyFont{\fontsize{8}{11}\helveticabold }
\def\firstAuthorLast{Needell {et~al.}} 
\def\Authors{Deanna Needell\,$^{1,*}$, Aaron A.~Nelson\,$^{2}$, Rayan Saab\,$^{3}$, Palina Salanevich\,$^{4}$,  and Olov Schavemaker\,$^{4}$}
\def\Address{$^{1}$Department of Mathematics, University of California, Los Angeles, CA, USA \\
$^{2}$Department of Mathematical Sciences, United States Air Force Academy, CO, USA \\ 
$^{3}$Department of Mathematics and Hal{\i}c{\i}o{\u{g}}lu Data Science Institute, University of California, San Diego, CA, USA \\
$^{4}$Mathematical Institute, Utrecht University, the Netherlands}

\begin{document}
\onecolumn
\firstpage{1}

\title {Random Vector Functional Link Networks for Function Approximation on Manifolds\footnote{The views expressed in this article are those of the authors and do not reflect the official policy or position of the U.S.~Air Force, Department of Defence, or U.S.~Government.}} 

\author[\firstAuthorLast ]{\Authors} 
\address{} 
\correspondance{} 

\extraAuth{}

\maketitle

\begin{abstract}

\section{}

The learning speed of feed-forward neural networks is notoriously slow and has presented a bottleneck in deep learning applications for several decades.
    For instance, gradient-based learning algorithms, which are used extensively to train neural networks, tend to work slowly when all of the network parameters must be iteratively tuned.
    To counter this, both researchers and practitioners have tried introducing randomness to reduce the learning requirement.
    Based on the original construction of Igelnik and Pao, single layer neural-networks with random input-to-hidden layer weights and biases have seen success in practice, but the necessary theoretical justification is lacking.
    In this paper, we begin to fill this theoretical gap.
    We then extend this result to the non-asymptotic setting using a concentration inequality for Monte-Carlo integral approximations.
    We provide a (corrected) rigorous proof that the Igelnik and Pao construction is a universal approximator for continuous functions on compact domains, with approximation error squared decaying asymptotically like \(O(1/n)\) for the number \(n\) of network nodes. We then extend this result to the non-asymptotic setting, proving that one can achieve any desired approximation error with high probability provided \(n\) is sufficiently large.
    We further adapt this randomized neural network architecture to approximate functions on smooth, compact submanifolds of Euclidean space, providing theoretical guarantees in both the asymptotic and non-asymptotic forms.
    Finally, we illustrate our results on manifolds with numerical experiments.

\tiny
 \keyFont{ \section{Keywords:} Machine learning, feed-forward neural networks, function approximation, smooth manifold, Random Vector Functional Link} 
\end{abstract}

\section{Introduction}

In recent years, neural networks have once again triggered an increased interest  among researchers in the machine learning community. So-called deep neural networks model functions using a composition of multiple hidden layers, each transforming (possibly non-linearly) the previous layer before building a final output representation, see  \citep{krizhevsky2012imagenet,szegedy2015going,he2016deep,huang2017densely,yang2018convolutional}. 
In machine learning parlance, these layers are determined by sets of \emph{weights} and \emph{biases} that can be tuned so that the network mimics the action of a complex function. In particular, a single layer feed-forward neural network (SLFN) with \(n\) nodes may be regarded as a parametric function \(f_n\colon\mathbb{R}^N\rightarrow\mathbb{R}\) of the form
\[f_n(x)=\sum_{k=1}^nv_k\rho(\langle w_k,x\rangle+b_k),
    \quad x\in\mathbb{R}^N.\]
Here, the function \(\rho\colon\mathbb{R}\rightarrow\mathbb{R}\) is called an activation function and is potentially nonlinear. Some typical examples include the sigmoid function $\rho(z)=\frac{1}{1+\exp(-z)}$, ReLU $\rho(z)=\max\{0,z\}$, and sign functions, among many others. 
The parameters of the SLFN are the number of nodes \(n\in\mathbb{N}\) in the the hidden layer, the input-to-hidden layer weights and biases \(\{w_k\}_{k=1}^n\subset\mathbb{R}^N\) and \(\{b_k\}_{k=1}^n\subset\mathbb{R}\) (resp.), and the hidden-to-output layer weights \(\{v_k\}_{k=1}^n\subset\mathbb{R}\). In this way, neural networks are fundamentally parametric families of functions whose parameters may be chosen to approximate a given function. 

It has been shown that any compactly supported continuous function can be approximated with any given precision by a single layer neural network with a suitably chosen number of nodes \citep{barron1993universal}, and
 harmonic analysis techniques have been used to study stability of such approximations \citep{candes1999harmonic}. Other recent results that take a different approach  directly analyze the capacity of neural networks from a combinatorial point of view \citep{vershynin2020memory,baldi2019capacity}.
 
While these results ensure existence of a neural network approximating a  function, practical applications require construction of such an approximation.
The parameters of the neural network can be chosen using optimization techniques to minimize the difference between the network and the function \(f\colon\mathbb{R}^N\rightarrow\mathbb{R}\) it is intended to model.
In practice, the function \(f\) is usually not known, and we only have access to a set \(\{(x_k,f(x_k))\}_{k=1}^m\) of values of the function at finitely many points sampled from its domain, called a \emph{training set}.
The approximation error can be measured by comparing the training data to the corresponding network outputs when evaluated on the same set of points, and the parameters of the neural network \(f_n\) can be  \emph{learned} by minimizing a given loss function \(\mathcal{L}(x_1,\ldots,x_m)\); a typical loss function is the sum-of-squares error
\[\mathcal{L}(x_1,\ldots,x_m)
    =\frac{1}{m}\sum_{k=1}^m|f(x_k)-f_n(x_k)|^2.\]
The SLFN which approximates \(f\) is then determined using an optimization algorithm, such as back-propagation, to find the network parameters which minimize \(\mathcal{L}(x_1,\ldots,x_m)\).
It is known that there exist weights and biases which make the loss function vanish when the number of nodes \(n\) is at least  \(m\), provided the activation function is bounded, nonlinear, and has at least one finite limit at either \(\pm\infty\)~\citep{huang1998upper}.

Unfortunately, optimizing the parameters in SLFNs can be difficult.
For instance, any non-linearity in the activation function can cause back-propagation to be very time consuming or get caught in local minima of the loss function~\citep{Suganthan2018LetterON}.
Moreover, deep neural networks can require massive amounts of training data, and so are typically unreliable for applications with very limited data availability, such as agriculture, healthcare, and ecology~\citep{olson2018modern}.

To address some of the difficulties associated with training deep neural networks, both researchers and practitioners have attempted to incorporate randomness in some way.
Indeed, randomization-based neural networks that yield closed form solutions typically require less time to train and avoid some of the pitfalls of traditional neural networks trained using back-propagation~\citep{Suganthan2018LetterON,201708,TEBRAAKE199571}.
One of the popular randomization-based neural network architectures is the Random Vector Functional Link (RVFL) network~\citep{Pao1992Functional,igelnik1995stochastic}, which is a single layer feed-forward neural network in which the input-to-hidden layer weights and biases are selected randomly and independently from a suitable domain and the remaining hidden-to-output layer weights are learned using training data.

By eliminating the need to optimize the input-to-hidden layer weights and biases, RVFL networks turn supervised learning into a purely linear problem.
To see this, define \({\rho(X)\in\mathbb{R}^{n\times m}}\) to be the matrix whose \(j\)th column is \(\{\rho(\langle w_k,x_j\rangle+b_k)\}_{k=1}^n\) and \(f(X)\in\mathbb{R}^m\) the vector whose \(j\)th entry is \(f(x_j)\).
Then the vector \(v\in\mathbb{R}^n\) of hidden-to-output layer weights is the solution to the matrix-vector equation \(f(X)=\rho(X)^Tv\), which can be solved by computing the Moore-Penrose pseudoinverse of \(\rho(X)^T\).

Although originally considered in the early- to mid-1990s~\citep{Pao1992Functional,pao1994learning,igelnik1995stochastic,pao1995functionalLINK}, RVFL networks have had much more recent success in several modern applications, including time-series data prediction~\citep{chen1999rapid}, handwritten word recognition~\citep{park2000unconstrained}, visual tracking~\citep{Zhang2016Visual}, signal classification~\citep{zhang2017benchmarking,katuwal20181ensemble}, regression~\citep{vukovic2018comprehensive}, and forecasting~\citep{tang2018noniterative,dash2018indian}.
Deep neural network architectures based on RVFL networks have also made their way into more recent literature~\citep{henriquez2018twitter,katuwal2019random}, although traditional, single layer RVFL networks tend to perform just as well as, and with lower training costs than, their multi-layer counterparts~\citep{katuwal2019random}.

Even though RVFL networks are proving their usefulness in practice, the supporting theoretical framework is currently lacking~\citep[see][]{zhang2019unsupervised}.
Most theoretical research into the approximation capabilities of deep neural networks centers around two main concepts: universal approximation of functions on compact domains and point-wise approximation on finite training sets. 
For instance, in the early 1990s it was shown that multi-layer feed-forward neural networks having activation functions that are continuous, bounded, and non-constant are universal approximators (in the \(L^p\) sense for \(1\leq p<\infty\)) of continuous functions on compact domains~\citep{hornik1991approximation, leshno1993multilayer}.
The most notable result in the existing literature regarding the universal approximation capability of RVFL networks is due to B.~Igelnik and Y.H.~Pao in the mid-1990s, who showed that such neural networks can universally approximate continuous functions on compact sets~\citep{igelnik1995stochastic};
the noticeable lack of results since has left a sizable gap between theory and practice. In this paper, we begin to bridge this gap by further improving the Igelnik and Pao result, and bringing the mathematical theory behind RFVL networks into the modern spotlight. Below, we introduce the notation that will be used throughout this paper, and describe our main contributions.

\subsection{Notation}

For a function \(f\colon\mathbb{R}^N\rightarrow\mathbb{R}\), the set \(\mathrm{supp}(f)\subset\mathbb{R}^N\) denotes the support of \(f\).
We denote by \(C_c(\mathbb{R}^N)\) and \(C_0(\mathbb{R}^N)\) the classes of continuous functions mapping \(\mathbb{R}^N\) to \(\mathbb{R}\) whose support sets are compact and vanish at infinity, respectively.
Given a set \(S\subset\mathbb{R}^N\), we define its radius to be \(\mathrm{rad}(S):=\sup_{x\in S}\Vert x\Vert_2\);
moreover, if \(\mathrm{d}\mu\) denotes the uniform volume measure on \(S\), then we write \(\mathrm{vol}(S):=\int_S\mathrm{d}\mu\) to represent the volume of \(S\).
For any probability distribution \(P\colon\mathbb{R}^N\rightarrow[0,1]\), a random variable \(X\) distributed according to \(P\) is denoted by \(X\sim P\), and we write its expectation as \(\mathbb{E}X:=\int_{\mathbb{R}^N}X\mathrm{d}P\).
The open \(\ell_p\) ball of radius \(r>0\) centered at \(x\in\mathbb{R}^N\) is denoted by \(B_p^N(x,r)\) for all \(1\leq p\leq\infty\); the \(\ell_p\) unit-ball centered at the origin is abbreviated \(B_p^N\).
Given a fixed \(\delta>0\) and a set \(S\subset\mathbb{R}^N\), a minimal \(\delta\)-net for \(S\), which we denote \(\mathcal{C}(\delta,S)\), is the smallest subset of \(S\) satisfying \(S\subset\cup_{x\in\mathcal{C}(\delta,S)}B_2^N(x,\delta)\); the \(\delta\)-covering number of \(S\) is the cardinality of a minimal \(\delta\)-net for \(S\) and is denoted \(\mathcal{N}(\delta,S):=|\mathcal{C}(\delta,S)|\).

\subsection{Main results}\label{sec: intro results}

In this paper, we study the uniform approximation capabilities of RVFL networks. More specifically, we consider the problem of using RVFL networks to estimate a continuous, compactly supported function on \(N\)-dimensional Euclidean space.

The first theoretical result on approximating properties of RVFL networks, due to Igelnik and Pao, guarantees that continuous functions may be universally approximated on compact sets using RVFL networks, provided the number of nodes \(n\in\mathbb{N}\) in the network goes to infinity \citep{igelnik1995stochastic}.
Moreover, it shows that the mean square error of the approximation vanishes at a rate proportional to \(1/n\).
At the time, this result was state-of-the-art and justified how RVFL networks were used in practice.
However, the original theorem is not technically correct.
In fact, several aspects of the proof technique are flawed.
Some of the minor flaws are mentioned in~\cite{li1997commentson}, but the subsequent revisions do not address the more significant issues which would make the statement of the result technically correct. We address these issues in this paper, see Remark~\ref{rmk: after IP95}.
Thus, our first contribution to the theory of RVFL networks is a corrected version of the original Igelnik and Pao theorem:

\begin{theorem}[\cite{igelnik1995stochastic}]\label{thm: IP95_short}
Let \(f\in C_c(\mathbb{R}^N)\) with \(K:=\mathrm{supp}(f)\) and fix any activation function $\rho$, such that either \OPtext{$\rho\in L^1(\mathbb{R})\cap L^\infty(\mathbb{R})$} with \(\int_{\mathbb{R}}\rho(z)\mathrm{d}z=1\) or $\rho$ is differentiable with \OPtext{$\rho'\in L^1(\mathbb{R})\cap L^\infty(\mathbb{R})$} and \(\int_{\mathbb{R}}\rho^{\prime}(z)\mathrm{d}z=1\). For any \(\varepsilon>0\), there exist  distributions from which input weights \(\{w_k\}_{k=1}^n\) and biases \(\{b_k\}_{k=1}^n\) are drawn, and there exist  hidden-to-output layer weights \(\{v_k\}_{k=1}^n\subset\mathbb{R}\) that depend on the realization of weights and biases, such that the sequence of RVFL networks \(\{f_n\}_{n=1}^{\infty}\) defined by
\[f_n(x):=\sum_{k=1}^nv_k\rho(\langle w_k,x\rangle+b_k)
    \quad\text{ for \(x\in K\)}\]
satisfies \[\mathbb{E}\int_K|f(x)-f_n(x)|^2\mathrm{d}x
    <\varepsilon+O(1/n),\]
as $n\to\infty.$

\end{theorem}

For a more precise formulation of Theorem~\ref{thm: IP95_short} and its proof, we refer the reader to Theorem~\ref{thm: IP95} and  Section~\ref{sec: proof thm IP95}. 

\begin{remark}\label{rmk: after IP95}
\mbox{}
\begin{enumerate}
    \item Even though in Theorem~\ref{thm: IP95_short} we only claim existence of the distribution for input weights \(\{w_k\}_{k=1}^n\) and biases \(\{b_k\}_{k=1}^n\), such a distribution is actually constructed in the proof. Namely, for any \(\varepsilon>0\), there exist constants \(\alpha,\Omega>0\) such that the random variables
\begin{linenomath}
\begin{align*}
    w_k&\sim \mathrm{Unif}([-\alpha\Omega,\alpha\Omega]^N);\\
    y_k&\sim \mathrm{Unif}(K);\\
    u_k&\sim \mathrm{Unif}([-\tfrac{\pi}{2}(2L+1),\tfrac{\pi}{2}(2L+1)]),
    \quad\text{where \(L:=\lceil\tfrac{2N}{\pi}\mathrm{rad}(K)\Omega-\tfrac{1}{2}\rceil\)},
\end{align*}
\end{linenomath}
 are independently drawn from their associated distributions, and  $b_k:=-\langle w_k,y_k\rangle-\alpha u_k$.
 
\item We note that, unlike the original theorem statement in~\cite{igelnik1995stochastic}, Theorem~\ref{thm: IP95_short} does not show exact convergence of the sequence of constructed RVFL networks \(f_n\) to the original function \(f\).
Indeed, it only ensures that the limit \(f_n\) is \(\varepsilon\)-close to \(f\).
This should still be sufficient for practical applications since, given a desired accuracy level \(\varepsilon>0\), one can find values of \(\alpha,\Omega,n\) such that this accuracy level is achieved on average.
Exact convergence can be proved if one replaces \(\alpha\) and \(\Omega\) in the distribution described above by sequences \(\{\alpha_n\}_{n=1}^{\infty}\) and \(\{\Omega_n\}_{n=1}^{\infty}\) of positive numbers, both tending to infinity with \(n\).
In this setting, however, there is no guaranteed rate of convergence;
moreover, as \(n\) increases, the ranges of the random variables \(\{w_k\}_{k=1}^n\) and \(\{u_k\}_{k=1}^n\) become increasingly larger, which may cause problems in practical applications.

\item The approach we take to construct the RVFL network approximating a function $f$ allows one to compute the output weights $\{v_k\}_{k=1}^n$ exactly (once the realization of random parameters is fixed), in the case where the function $f$ is known. For the details, we refer the reader to equations~\eqref{eqn: F-b} and~\eqref{eqn: I-P RVFLs} in the proof of Theorem~\ref{thm: IP95_short}. If we only have access to a training set that is sufficiently large and uniformly distributed over the support of $f$, these formulas can be used to compute the output weights approximately, instead of solving the least squares problem.

\item Note that the normalization $\int_\mathbb{R}\rho(z)\mathrm{d}z=1$ of the activation function can be replaced by the condition $\int_\mathbb{R}\rho(z)\mathrm{d}z\ne 0$. Indeed, in the case when $\rho\in L^1(\mathbb{R})\cap L^\infty(\mathbb{R})$ and $\int_\mathbb{R}\rho(z)\mathrm{d}z\notin\{0,1\},$ one can simply use Theorem~\ref{thm: IP95_short} to approximate $\frac{1}{\int_\mathbb{R}\rho(z)\mathrm{d}z}f$ by a sequence of RVFL network with the activation function~$\frac{1}{\int_\mathbb{R}\rho(z)\mathrm{d}z}\rho$. Mutatis mutandis in the case when $\int_\mathbb{R}\rho'(z)\mathrm{d}z'\notin\{0,1\}.$ More generally, this trick allows any of our theorems to be applied in the case $\int_\mathbb{R}\rho(z)\mathrm{d}z\ne 0.$
\end{enumerate}
\end{remark}

One of the drawbacks of Theorem~\ref{thm: IP95_short} is that the mean square error guarantee is asymptotic in the number of nodes used in the neural network.
This is clearly impractical for applications, and so it is desirable to have a more explicit error bound for each fixed number \(n\) of nodes used.
To this end, we provide a new, non-asymptotic version of Theorem~\ref{thm: IP95_short}, which provides an error guarantee with high probability whenever the number of network nodes is large enough, albeit at the price of an additional Lipschitz requirement on the activation function:

\begin{theorem}\label{thm: probabilistic Igelnik-Pao_short}
Let \(f\in C_c(\mathbb{R}^N)\) with \(K:=\mathrm{supp}(f)\) and fix any activation function \OPtext{${\rho\in L^1(\mathbb{R})\cap L^\infty(\mathbb{R})}$} with $\int_\mathbb{R}\rho(z)\mathrm{d}z=1.$
Suppose further that \(\rho\) is \(\kappa\)-Lipschitz on \(\mathbb{R}\) for some \(\kappa>0\).
For any \(\varepsilon>0\) and \(\eta\in(0,1)\), suppose that 
\OPtext{$n\ge C(N,f,\rho) \varepsilon^{-1}\log(\eta^{-1}/\varepsilon),$} where \OPtext{$C(N,f,\rho)$} is independent of $\varepsilon$ and $\eta$ and depends on $f$, $\rho$, and superexponentially on~$N$. Then there exist  distributions from which input weights \(\{w_k\}_{k=1}^n\) and biases \(\{b_k\}_{k=1}^n\) are drawn, and there exist  hidden-to-output layer weights \(\{v_k\}_{k=1}^n\subset\mathbb{R}\) that depend on the realization of weights and biases, such that the RVFL network defined by \[f_n(x):=\sum_{k=1}^nv_k\rho(\langle w_k,x\rangle+b_k)
    \quad\text{ for \(x\in K\)}\]
satisfies \[\int_K|f(x)-f_n(x)|^2\mathrm{d}x
    <\varepsilon\]
with probability at least \(1-\eta\). 
\end{theorem}

For simplicity, the bound on the number $n$ of the nodes on the hidden layer here is rough\OPtext{.}
For a more precise formulation of this result that contains a bound with explicit constant, we refer the reader to Theorem~\ref{thm: probabilistic Igelnik-Pao} in Section~\ref{sec: proof cor IP95 thm probabilistic IP}. We also note that the distribution of the input weight and bias here can be selected as described in Remark~\ref{rmk: after IP95}.

\longhide{Note that for small \(\varepsilon>0\), a Taylor expansion shows that \(\log(1+\varepsilon)=\varepsilon+O(\varepsilon^2)\), and so the requirement on the number of nodes in Theorem~\ref{thm: probabilistic Igelnik-Pao} behaves like \(n\gtrsim\varepsilon^{-2}\log(\eta^{-1}\mathcal{N}(\delta,K))\).}

The constructions of RVFL networks presented in Theorems~\ref{thm: IP95_short} and~\ref{thm: probabilistic Igelnik-Pao_short} depend heavily on the dimension of the ambient space \(\mathbb{R}^N\). 
If $N$ is small, this dependence does not present much of a problem.
However, many modern applications require the ambient dimension to be large.
Fortunately, a common assumption in practice is that support of the signals of interest lie on a lower-dimensional manifold embedded in \(\mathbb{R}^N\).
For instance, the landscape of cancer cell states can be modeled using nonlinear, locally continuous ``cellular manifolds;" indeed, while the ambient dimension of this state space is typically high (e.g., single-cell RNA sequencing must account for approximately 20,000 gene dimensions), cellular data actually occupies an intrinsically lower dimensional space~\cite{burkhardt2022mapping}.
Likewise, while the pattern space of neural population activity in the brain is described by an exponential number of parameters, the spatiotemporal dynamics of brain activity lie on a lower-dimensional subspace or ``neural manifold"~\cite{mitchell-heggs2023neural}.
In this paper, we propose a new RVFL network architecture for approximating continuous functions defined on smooth compact manifolds that allows to replace the dependence on the ambient dimension $N$ with dependence on the manifold intrinsic dimension. We show that RVFL approximation results can be extended to this setting.
More precisely, we prove the following analog of Theorem~\ref{thm: probabilistic Igelnik-Pao_short}. 

\begin{theorem}\label{thm: manifold Igelnik-Pao_short}


Let \(\mathcal{M}\subset\mathbb{R}^N\) be a smooth, compact \(d\)-dimensional manifold with finite atlas \(\{(U_j,\phi_j)\}_{j\in J}\) and \(f\in C(\mathcal{M})\).
Fix any activation function \OPtext{$\rho\in L^1(\mathbb{R})\cap L^\infty(\mathbb{R})$} with \(\int_{\mathbb{R}}\rho(z)\mathrm{d}z=1\) such that \(\rho\) is \(\kappa\)-Lipschitz on \(\mathbb{R}\) for some \(\kappa>0\).
For any \(\varepsilon>0\) and \(\eta\in(0,1)\), suppose \OPtext{$n\ge C(d,f,\rho) \varepsilon^{-1}\log(\eta^{-1}/\varepsilon),$} where \OPtext{$C(d,f,\rho)$} is independent of $\varepsilon$ and $\eta$ and depends on $f$, $\rho$, and superexponentially on~$d$.
Then there exists an RVFL-like approximation $f_n$ of the function $f$ with a parameter selection similar to the Theorem~\ref{thm: IP95_short} construction that satisfies \[\int_{\mathcal{M}}|f(x)-f_n(x)|^2\mathrm{d}x
    <\varepsilon\]
with probability at least \(1-\eta\).
\end{theorem}
 
 For a the construction of the RVFL-like approximation $f_n$ and a more precise formulation of this result and an analog of Theorem~\ref{thm: IP95_short} applied to manifolds, we refer the reader to Section~\ref{sec: man construction} and Theorems~\ref{thm: manifold Igelnik-Pao}~and~\ref{thm: probabilistic manifold Igelnik-Pao}. We note that the approximation $f_n$ here is not obtained as a single RVFL network construction, but rather as a combination of several RVFL networks in local manifold coordinates.

\subsection{Organization}
The remaining part of the paper is organized as follows. In Section \ref{sec: preliminaries}, we discuss some theoretical preliminaries on concentration bounds for Monte-Carlo integration and on smooth compact manifolds. Monte-Carlo integration is an essential ingredient in our construction of RVFL networks approximating a given function, and we use the results listed in this section to establish approximation error bounds. Theorem~\ref{thm: IP95_short} is proven in  Section~\ref{sec: proof thm IP95}, where we break down the proof into four main steps, constructing a limit-integral representation of the function to be approximated in Lemmas~\ref{lem: Step 1}~and~\ref{lem: step 2}, then using Monte-Carlo approximation of the obtained integral to construct an RVFL network in Lemma~\ref{lem: part 3}, and, finally, establishing approximation guarantees for the constructed RVFL network. The proofs of Lemmas~\ref{lem: Step 1},~\ref{lem: step 2},~and~\ref{lem: part 3} can be found in Sections~\ref{sec: step 1},~\ref{sec: step 2}, ~and~\ref{sec: step 3}, respectively. We further study properties of the constructed RVFL networks and prove the non-asymptotic approximation result of Theorem~\ref{thm: probabilistic Igelnik-Pao_short} in Section~\ref{sec: proof cor IP95 thm probabilistic IP}.
In Section~\ref{sec: manifold results}, we generalize our results and propose a new RVFL network architecture for approximating continuous functions defined on smooth compact manifolds. We show that RVFL approximation results can be extended to this setting by proving an analog of Theorem~\ref{thm: IP95_short} in Section~\ref{sec: manifold IP95 proof} and Theorem~\ref{thm: manifold Igelnik-Pao_short} in Section~\ref{sec: Manifold proof}. Finally, in Section~\ref{sec: numerics}, we provide numerical evidence to illustrate the result of Theorem~\ref{thm: manifold Igelnik-Pao_short}. 


\section{Materials and Methods}\label{sec: preliminaries}

In this section, we briefly introduce supporting material and theoretical results which we will need in later sections.
This material is far from exhaustive, and is meant to be a survey of definitions, concepts, and key results.

\subsection{A concentration bound for classic Monte-Carlo integration}\label{sec: monte-carlo}

A crucial piece of the proof technique employed in~\cite{igelnik1995stochastic}, which we will use repeatedly, is the use of the Monte-Carlo method to approximate high-dimensional integrals. As such, we start with the background on Monte-Carlo integration. The following introduction is adapted from the material in \cite{dick2013high}.

Let $f\colon\mathbb{R}^N\rightarrow\mathbb{R}$ and $S\subset\mathbb{R}^N$ a compact set.
Suppose we want to estimate the integral \({I(f,S):=\int_Sf\mathrm{d}\mu}\), where \(\mu\) is the uniform measure on \(S\).
The classic Monte Carlo method does this by an equal-weight cubature rule, \[I_n(f,S):=\frac{\mathrm{vol}(S)}{n}\sum_{j=1}^nf(x_j),\]
where \(\{x_j\}_{j=1}^n\) are independent identically distributed uniform random samples from \(S\) and \linebreak\({\mathrm{vol}(S):=\int_Sd\mu}\) is the volume of \(S\).
In particular, note that \(\mathbb{E}I_n(f,S)=I(f,S)\) and \[\mathbb{E}I_n(f,S)^2
    =\frac{1}{n}\big(\mathrm{vol}(S)I(f^2,S)+(n-1)I(f,S)^2\big).\]
Let us define the quantity
\begin{linenomath}
\begin{align}\label{eqn: MC variance}
    \sigma(f,S)^2:=\frac{I(f^2,S)}{\mathrm{vol}(S)}-\frac{I(f,S)^2}{\mathrm{vol}^2(S)}.
\end{align}
\end{linenomath}
It follows that the random variable \(I_n(f)\) has mean \(I(f,S)\) and variance \(\mathrm{vol}^2(S)\sigma(f,S)^2/n\).
Hence, by the Central Limit Theorem, provided that \(0<\mathrm{vol}^2(S)\sigma(f,S)^2<\infty\), we have \[\lim_{n\rightarrow\infty}\mathbb{P}\Big(|I_n(f,S)-I(f,S)|\leq \frac{C\varepsilon(f,S)}{\sqrt{n}}\Big)=(2\pi)^{-1/2}\int_{-C}^Ce^{-x^2/2}\mathrm{d}x\]
for any constant \(C>0\), where $\varepsilon(f,S) := {\mathrm{vol}(S)\sigma(f, S)}$. 
This yields the following well-known result:

\begin{theorem}\label{thm: MC MSE}
For any \(f\in L^2(S,\mu)\), the mean-square error of the Monte Carlo approximation \(I_n(f,S)\) satisfies \[\mathbb{E}\big|I_n(f,S)-I(f,S)\big|^2    =\frac{\mathrm{vol}^2(S)\sigma(f,S)^2}{n},\]
where the expectation is taken with respect to the random variables \(\{x_j\}_{j=1}^n\) and \(\sigma(f,S)\) is defined in~\eqref{eqn: MC variance}.
\end{theorem}

\noindent In particular, Theorem~\ref{thm: MC MSE} implies \(\mathbb{E}\big|I_n(f,S)-I(f,S)\big|^2=O(1/n)\) as $n\to\infty.$

In the non-asymptotic setting, we are interested in obtaining a useful bound on the probability \({\mathbb{P}(|I_n(f,S)-I(f,S)|\geq t)}\) for all \(t>0\). The following lemma follows from a generalization of Bennett's inequality (Theorem~7.6 in~\cite{ledoux2001concentration}; see also~\cite{massart1998constants,talagrand1996new}).

\begin{lemma}\label{lem: MC concentration ineq}
For any \(f\in L^2(S)\) and \(n\in\mathbb{N}\) we have \[\mathbb{P}\Big(|I_n(f,S)-I(f,S)|\geq t\Big)
    \leq3\exp\left(-\frac{nt}{CK}\log\left(1+\frac{Kt}{\displaystyle\mathrm{vol}(S)I(f^2,S)}\right)\right)\]
for all \(t>0\) and a universal constant \(C>0\), provided \OPtext{$\lvert\mathrm{vol}(S)f(x)\rvert\leq K$} for almost every \(x\in S\).
\end{lemma}

\subsection{Smooth, compact manifolds in Euclidean space}\label{sec: manifold basics}

In this section we review several concepts of smooth manifolds that will be useful to us later.
Many of the definitions and results that follow can be found, for instance, in~\cite{shaham2018provable}.
Let \(\mathcal{M}\subset\mathbb{R}^N\) be a smooth, compact \(d\)-dimensional manifold.
A \emph{chart} for \(\mathcal{M}\) is a pair \((U,\phi)\) such that \(U\subset\mathcal{M}\) is an open set and \(\phi\colon U\rightarrow\mathbb{R}^d\) is a homeomorphism.
One way to interpret a chart is as a tangent space at some point \(x\in U\); in this way, a chart defines a Euclidean coordinate system on \(U\) via the map \(\phi\).
A collection \(\{(U_j,\phi_j)\}_{j\in J}\) of charts defines an \emph{atlas} for \(\mathcal{M}\) if \(\cup_{j\in J}U_j=\mathcal{M}\). We now define a special collection of functions on \(\mathcal{M}\) called a \emph{partition of unity}.
\begin{definition}\label{def: partition of unity}
Let \(\mathcal{M}\subset\mathbb{R}^N\) be a smooth manifold.
A \emph{partition of unity} of \(\mathcal{M}\) with respect to an open cover \(\{U_j\}_{j\in J}\) of \(\mathcal{M}\) is a family of nonnegative smooth functions \(\{\eta_j\}_{j\in J}\) such that for every \(x\in\mathcal{M}\) we have \(1=\sum_{j\in J}\eta_j(x)\) and, for every \(j\in J\), \(\mathrm{supp}(\eta_j)\subset U_j\).
\end{definition}
\noindent It is known that if \(\mathcal{M}\) is compact there exists a partition of unity of \(\mathcal{M}\) such that \(\mathrm{supp}(\eta_j)\) is compact for all \(j\in J\) \citep[see][]{tu2010introduction}.
In particular, such a partition of unity exists for any open cover of \(\mathcal{M}\) corresponding to an atlas.

Fix an atlas \(\{(U_j,\phi_j)\}_{j\in J}\) for \(\mathcal{M}\), as well as the corresponding, compactly supported partition of unity \(\{\eta_j\}_{j\in J}\).
Then we have the following, useful result \citep[see][Lemma~4.8]{shaham2018provable}.
\begin{lemma}\label{lem: manifold funct rep}
Let \(\mathcal{M}\subset\mathbb{R}^N\) be a smooth, compact manifold with atlas \(\{(U_j,\phi_j)\}_{j\in J}\) and compactly supported partition of unity \(\{\eta_j\}_{j\in J}\).
For any \(f\in C(\mathcal{M})\) we have \[f(x)=\sum_{\{j\in J\colon x\in U_j\}}(\hat{f}_j\circ\phi_j)(x)\]
for all \(x\in\mathcal{M}\), where \[\hat{f}_j(z):=\begin{cases}
    f(\phi_j^{-1}(z))\,\eta_j(\phi_j^{-1}(z))\quad&z\in\phi_j(U_j)\\
    0&\text{otherwise}.
    \end{cases}\]
\end{lemma}

In later sections, we use the representation of Lemma~\ref{lem: manifold funct rep} to integrate functions \(f\in C(\mathcal{M})\) over \(\mathcal{M}\).
To this end, for each \(j\in J\), let \(D\phi_j(y)\) denote the differential of \(\phi_j\) at \(y\in U_j\), which is a map from the tangent space \(T_y\mathcal{M}\) into \(\mathbb{R}^d\).
One may interpret \(D\phi_j(y)\) as the matrix representation of a basis for the cotangent space at \(y\in U_j\).
As a result, \(D\phi_j(y)\) is necessarily invertible for each \(y\in U_j\), and so we know that \(|\det(D\phi_j(y))|>0\) for each \(y\in U_j\).
Hence, it follows by the change of variables theorem that
\begin{linenomath}
\begin{align}\label{eqn: man CoVs}
    \int_{\mathcal{M}}f(x)\mathrm{d}x
    =\int_{\mathcal{M}}\sum_{\{j\in J\colon x\in U_j\}}(\hat{f}_j\circ\phi_j)(x)\mathrm{d}x
    =\sum_{j\in J}\int_{\phi_j(U_j)}\frac{\hat{f}_j(z)}{|\det(D\phi_j(\phi_j^{-1}(z)))|}\mathrm{d}z.
\end{align}
\end{linenomath}

\section{Results}

In this section, we prove our main results formulated in Section~\ref{sec: intro results} and also use numerical simulations to illustrate the RVFL approximation performance in a low-dimensional submanifold setup. To improve  readability of this section, we postpone the proofs of technical lemmas till Section~\ref{sec:app}.

\subsection{Proof of Theorem~\ref{thm: IP95_short}}\label{sec: proof thm IP95}

We split the proof of the theorem into two parts, the first handling the case \OPtext{$\rho\in L^1(\mathbb{R})\cap L^\infty(\mathbb{R})$} and the second, addressing the case \OPtext{$\rho'\in L^1(\mathbb{R})\cap L^\infty(\mathbb{R}).$} 

\subsubsection{Proof of Theorem \ref{thm: IP95_short} when \OPtext{$\rho\in L^1(\mathbb{R})\cap L^\infty(\mathbb{R})$}}
We begin by restating the theorem in a form that explicitly includes the distributions that we draw our random variables from.

\begin{theorem}[\cite{igelnik1995stochastic}]\label{thm: IP95}
Let \(f\in C_c(\mathbb{R}^N)\) with \(K:=\mathrm{supp}(f)\) and fix any activation function \OPtext{$\rho\in L^1(\mathbb{R})\cap L^\infty(\mathbb{R})$} with \(\int_{\mathbb{R}}\rho(z)\mathrm{d}z=1\).
For any \(\varepsilon>0\), there exist constants \(\alpha,\Omega>0\) such that the following holds:
If, for $k\in\mathbb{N}$, the random variables
\begin{linenomath}
\begin{align*}
    w_k&\sim \mathrm{Unif}([-\alpha\Omega,\alpha\Omega]^N);\\
    y_k&\sim \mathrm{Unif}(K);\\
    u_k&\sim \mathrm{Unif}([-\tfrac{\pi}{2}(2L+1),\tfrac{\pi}{2}(2L+1)]),
    \quad\text{where \(L:=\lceil\tfrac{2N}{\pi}\mathrm{rad}(K)\Omega-\tfrac{1}{2}\rceil\)},
\end{align*}
\end{linenomath}
 are independently drawn from their associated distributions, and  \[b_k:=-\langle w_k,y_k\rangle-\alpha u_k,\] then there exist  hidden-to-output layer weights \(\{v_k\}_{k=1}^n\subset\mathbb{R}\) (that depend on the realization of the weights \(\{w_k\}_{k=1}^n\) and biases \(\{b_k\}_{k=1}^n\)) such that the sequence of RVFL networks \(\{f_n\}_{n=1}^{\infty}\) defined by \[f_n(x):=\sum_{k=1}^nv_k\rho(\langle w_k,x\rangle+b_k)
    \quad\text{ for \(x\in K\)}\]
satisfies \[\mathbb{E}\int_K|f(x)-f_n(x)|^2\mathrm{d}x
    \leq\varepsilon+O(1/n).\]
\OPtext{as $n\to\infty.$} 
\end{theorem}

\begin{proof} Our proof technique is based on that introduced by Igelnik and Pao, and can be divided into four steps. The first three steps essentially consist of Lemma \ref{lem: Step 1}, Lemma \ref{lem: step 2}, and Lemma \ref{lem: part 3}, and the final step combines them to obtain the desired result. 
First, the function \(f\) is approximated by
\OPtext{a convolution,}
given in Lemma \ref{lem: Step 1}. The proof of this result can be found in Section \ref{sec: step 1}.

\begin{lemma}\label{lem: Step 1}
Let \(f\in C_0(\mathbb{R}^N)\) and \OPtext{$h\in L^1(\mathbb{R}^N)$} with \OPtext{$\int_{\mathbb{R}^N}h(z)\mathrm{d}z=1.$}
For $\Omega>0$, define
\begin{linenomath}
\begin{align}\label{eqn: window transform}
    h_\Omega(y):=\Omega^N h(\Omega y).
\end{align}
\end{linenomath}
Then we have
\begin{linenomath}
\begin{align}\label{eqn: lim-int rep 2}
    f(x)
    =\lim_{\Omega\to\infty}(f*h_\Omega)(x)
\end{align}
\end{linenomath}
uniformly for all \(x\in\mathbb{R}^N\).
\end{lemma}


Next, we represent \(f\) as the limiting value of a multidimensional integral over the parameter space.
In particular, we replace \OPtext{$(f*h_\Omega)(x)$} in the convolution identity~\eqref{eqn: lim-int rep 2} with a function of the form  \(\int_KF(y)\rho(\langle w,x\rangle+b(y))\mathrm{d}y\), as this will introduce the RVFL structure we require.
To achieve this, we first use a truncated cosine function in place of the activation function \(\rho\) and then switch back to a general activation function.

To that end, for each fixed \(\Omega>0\), let \(L=L(\Omega):=\lceil\frac{2N}{\pi}\mathrm{rad}(K)\Omega-\frac{1}{2}\rceil\) and define \(\cos_{\Omega}\colon\mathbb{R}\rightarrow[-1,1]\) by
\begin{linenomath}
\begin{align}\label{eqn: trunc cos}
    \cos_{\Omega}(x):=\begin{cases}
    \cos(x)\qquad&x\in[-\tfrac{1}{2}(2L+1)\pi,\tfrac{1}{2}(2L+1)\pi],\\
    0&\text{otherwise}.
    \end{cases}
\end{align}
\end{linenomath}
Moreover, introduce the functions
\begin{linenomath}
\begin{align}\label{eqn: F-b}
    \begin{split}
    F_{\alpha,\Omega}(y,w,u)
    &:=\frac{\alpha}{(2\pi)^N}f(y)\cos_{\Omega}(u)\prod_{j=1}^N\phi(w(j)/\Omega),\\
    b_{\alpha}(y,w,u)
    &:=-\alpha(\langle w,y\rangle+u)
    \end{split}
\end{align}
\end{linenomath}
where \(y,w\in\mathbb{R}^N\) and \(u\in\mathbb{R}\) \OPtext{and $\phi=A*A$ for \emph{any} even function $A\in C^\infty(\mathbb{R})$ supported on $[-\tfrac{1}{2},\tfrac{1}{2}]$ s.t.\ $\lVert A\rVert_2=1.$} Then we have the following lemma, a detailed proof of which can be found in Section \ref{sec: step 2}.

\begin{lemma}\label{lem: step 2}
Let \(f\in C_c(\mathbb{R}^N)\) and \(\rho\in L^1(\mathbb{R})\) with \(K:=\mathrm{supp}(f)\) and \(\int_{\mathbb{R}}\rho(z)\mathrm{d}z=1\).
Define $F_{\alpha,\Omega}$ and \(b_{\alpha}\) as in~\eqref{eqn: F-b} for all
\OPtext{$\alpha>0.$}
Then, for \(L:=\lceil\frac{2N}{\pi}\mathrm{rad}(K)\Omega-\frac{1}{2}\rceil\), we have
\begin{linenomath}
\begin{align}\label{eqn: final lim-int rep}
    f(x)   =\lim_{\Omega\rightarrow\infty}\lim_{\alpha\rightarrow\infty}
    \int_{K(\Omega)}
    F_{\alpha,\Omega}(y,w,u)\rho\big(\alpha\langle w,x\rangle+b_{\alpha}(y,w,u)\big)\mathrm{d}y\mathrm{d}w\mathrm{d}u
\end{align}
\end{linenomath}
uniformly for every \(x\in K\), where \(K(\Omega):=K\times[-\Omega,\Omega]^N\times[-\frac{\pi}{2}(2L+1),\frac{\pi}{2}(2L+1)]\).
\end{lemma}

The next step in the proof of Theorem~\ref{thm: IP95} is to approximate the integral in~\eqref{eqn: final lim-int rep} using the Monte-Carlo method. Define
\OPtext{$v_k:=\frac{\mathrm{vol}(K(\Omega))}{n}F_{\alpha,\Omega}\Big(y_k,\frac{w_k}{\alpha},u_k\Big)$} for \(k=1,\ldots,n\), and the random variables \(\{f_n\}_{n=1}^{\infty}\) by
\begin{linenomath}
\begin{align}\label{eqn: I-P RVFLs}
    f_n(x)
    :=\sum_{k=1}^nv_k\rho\big(\langle w_k,x\rangle+b_k\big).
\end{align}
\end{linenomath}

Then, we have the following lemma that is proven in Section \ref{sec: step 3}.
\begin{lemma}\label{lem: part 3}
Let \(f\in C_c(\mathbb{R}^N)\) and \OPtext{$\rho\in L^1(\mathbb{R})\cap L^\infty(\mathbb{R})$} with \(K:=\mathrm{supp}(f)\) and \(\int_{\mathbb{R}}\rho(z)\mathrm{d}z=1\).
Then, as $n\to\infty$, we have
\begin{linenomath}
\begin{align}\label{eqn: RVFL rep}
   \mathbb{E}\int_K
    \left|\int_{K(\Omega)}
    F_{\alpha,\Omega}(y,w,u)\rho\big(\alpha\langle w,x\rangle+b_\alpha(y,w,u)\big)\mathrm{d}y\mathrm{d}w\mathrm{d}u
    -f_n(x)\right|^2\mathrm{d}x
    =O(1/n),
\end{align}
\end{linenomath}
where \(K(\Omega):=K\times[-\Omega,\Omega]^N\times[-\frac{\pi}{2}(2L+1),\frac{\pi}{2}(2L+1)]\) and \(L:=\lceil\frac{2N}{\pi}\mathrm{rad}(K)\Omega-\frac{1}{2}\rceil\).
\end{lemma}

To complete the proof of Theorem~\ref{thm: IP95} we combine the limit representation~\eqref{eqn: final lim-int rep} with the Monte-Carlo error guarantee~\eqref{eqn: RVFL rep} and show that, given any \(\varepsilon>0\), there exist \(\alpha,\Omega>0\) such that \[\mathbb{E}\int_K|f(x)-f_n(x)|^2\mathrm{d}x\leq\varepsilon+O(1/n)\]
\OPtext{as $n\to\infty.$} To this end, let \(\varepsilon^{\prime}>0\) be arbitrary and consider the integral \(I(x;p)\) given by
\begin{linenomath}
\begin{align}\label{eqn: MC int}
    I(x;p)
    :=\int_{K(\Omega)}
    \Big(F_{\alpha,\Omega}(y,w,u)\rho\big(\alpha\langle w,x\rangle+b_{\alpha}(y,w,u)\big)\Big)^p\mathrm{d}y\mathrm{d}w\mathrm{d}u
\end{align}
\end{linenomath}
for \(x\in K\) and \(p\in\mathbb{N}\).
By~\eqref{eqn: final lim-int rep}, there \OPtext{exist} \(\alpha,\Omega>0\) such that \(|f(x)-I(x;1)|<\varepsilon^{\prime}\) holds \if{uniformly}\fi for every \(x\in K\), and so it follows that \[\big|f(x)-f_n(x)\big|
    <\varepsilon^{\prime}+\big|I(x;1)-f_n(x)\big|\]
for every \(x\in K\).
\OPtext{Jensen's inequality now yields that}
\begin{linenomath}
\begin{align}\label{eqn: asympt error expansion}
    \mathbb{E}\int_K|f(x)-f_n(x)|^2\mathrm{d}x
    \leq2\mathrm{vol}(K)(\varepsilon^{\prime})^2
    +2\mathbb{E}\int_K\big|I(x;1)-f_n(x)\big|^2\mathrm{d}x.
\end{align}
\end{linenomath}
By~\eqref{eqn: RVFL rep}, we know that the second term on the right-hand side of~\eqref{eqn: asympt error expansion} 
\OPtext{is $O(1/n).$}
Therefore, we have \[\mathbb{E}\int_K|f(x)-f_n(x)|^2\mathrm{d}x\leq2\mathrm{vol}(K)(\varepsilon^{\prime})^2+O(1/n),\]
and so the proof is completed by taking \(\varepsilon^{\prime}=\sqrt{\varepsilon/2\mathrm{vol}(K)}\) and choosing \(\alpha,\Omega>0\) accordingly.
\end{proof}

\subsubsection{Proof of Theorem~\ref{thm: IP95_short} when $\rho' \in L^1(\mathbb{R})\cap L^\infty(\mathbb{R})$}
The full statement of the theorem is identical to that of Theorem~\ref{thm: IP95} albeit now with ${\rho' \in L^1(\mathbb{R})\cap L^\infty(\mathbb{R})}$, so we omit it for brevity. Its proof is also similar to the proof of the case where $\rho\in L^1(\mathbb{R}) \cap L^\infty(\mathbb{R})$ with some key modifications. Namely, one 
uses an integration by parts argument to modify the part of the proof corresponding to Lemma \ref{lem: step 2}. The details of this argument are presented in the appendix,  Section  \ref{sec: rho_prime_proof}.

\subsection{Proof of Theorem~\ref{thm: probabilistic Igelnik-Pao_short}}\label{sec: proof cor IP95 thm probabilistic IP}

In this section we prove the non-asymptotic result for RVFL networks in \(\mathbb{R}^N\), and we begin with a more precise statement of the theorem that makes all the dimensional dependencies explicit.

\begin{theorem}\label{thm: probabilistic Igelnik-Pao}
\if{Let \(f\in C_c(\mathbb{R}^N)\) with \(K:=\mathrm{supp}(f)\) and fix any activation function \(\rho\in L^1(\mathbb{R})\cap L^2(\mathbb{R})\).
Suppose further that \(\rho\) is \(\kappa\)-Lipschitz on~\(\mathbb{R}\) for some \(\kappa>0\).
For any \(\varepsilon>0\) and \(\eta\in(0,1)\), there exist constants \(\alpha,\Omega>0\) and hidden-to-output layer weights \(\{v_k\}_{k=1}^n\subset\mathbb{R}\) such that the following holds:
Suppose, for $k=1,...,n$, the random variables
\begin{align*}
    w_k&\sim \mathrm{Unif}([-\alpha\Omega,\alpha\Omega])^N;\\
    y_k&\sim \mathrm{Unif}(K);\\
    u_k&\sim \mathrm{Unif}([-\tfrac{\pi}{2}(2L+1),\tfrac{\pi}{2}(2L+1)]),
    \quad\text{where \(L:=\lceil\tfrac{2N}{\pi}\mathrm{rad}(K)\Omega-\tfrac{1}{2}\rceil\)},
\end{align*}
 are independently drawn from their associated distributions, and  $$b_k:=-\langle w_k,y_k\rangle-\alpha u_k.$$}\fi
 Consider the hypotheses of Theorem \ref{thm: IP95} and suppose further that \(\rho\) is \(\kappa\)-Lipschitz on \(\mathbb{R}\) for some \(\kappa>0\).
For any \[0
    <\delta
    <\frac{\sqrt{\varepsilon}}{8\sqrt{2N}\kappa\alpha^2M\Omega(\Omega/\pi)^N\mathrm{vol}^{3/2}(K)(\pi+2N\mathrm{rad}(K)\Omega)},\]
Suppose \[n \geq\frac{c\Sigma\alpha(\Omega/\pi)^N(\pi+2N\mathrm{rad}(K)\Omega)\log(3\eta^{-1}\mathcal{N}(\delta,K))}{\sqrt{\varepsilon}\log\big(1+\frac{\sqrt{\varepsilon}}{\Sigma\alpha(\Omega/\pi)^N(\pi+2N\mathrm{rad}(K)\Omega)}\big)},\]
where \(M:=\sup_{x\in K}|f(x)|\), \(c>0\) is a numerical constant, and \OPtext{$\Sigma$ is a constant} depending on \(f\) and \(\rho\), and let parameters \(\{w_k\}_{k=1}^n\), \(\{b_k\}_{k=1}^n\), and \(\{v_k\}_{k=1}^n\) be as in Theorem \ref{thm: IP95}. Then the RVFL network defined by \[f_n(x):=\sum_{k=1}^nv_k\rho(\langle w_k,x\rangle+b_k)
    \quad\text{ for \(x\in K\)}\]
satisfies \[\int_K|f(x)-f_n(x)|^2\mathrm{d}x
    <\varepsilon\]
with probability at least \(1-\eta\).
\end{theorem}

\begin{proof}
Let \(f\in C_c(\mathbb{R}^N)\) with \(K:=\mathrm{supp}(f)\) and suppose \(\varepsilon>0\), \(\eta\in(0,1)\) are fixed.
Take an arbitrarily \(\kappa\)-Lipschitz activation function \OPtext{$\rho\in L^1(\mathbb{R})\cap L^\infty(\mathbb{R}).$}
We wish to show that there exists an RVFL network \(\{f_n\}_{n=1}^{\infty}\) defined on \(K\) that \OPtext{satisfies} the \[\int_K|f(x)-f_n(x)|^2\mathrm{d}x<\varepsilon\]
with probability at least \(1-\eta\) when \(n\) is chosen sufficiently large. 
The proof is obtained by modifying the proof of Theorem~\ref{thm: IP95} for the asymptotic case.

We begin by repeating the first two steps in the proof of Theorem~\ref{thm: IP95} from Sections~\ref{sec: step 1} and~\ref{sec: step 2}.
In particular, by Lemma~\ref{lem: step 2} we have the representation~\eqref{eqn: final lim-int rep}, namely, \[f(x)=\lim_{\Omega\rightarrow\infty}\lim_{\alpha\rightarrow\infty}\int_{K(\Omega)}F_{\alpha,\Omega}(y,w,u)\rho\big(\alpha\langle w,x\rangle+b_{\alpha}(y,w,u)\big)\mathrm{d}y\mathrm{d}w\mathrm{d}u\]
holds uniformly for all \(x\in K\).
Hence, if we define the random variables \(f_n\) and \(I_n\) from Section~\ref{sec: step 3} as in~\eqref{eqn: I-P RVFLs} and~\eqref{eqn: MC sum}, respectively, we seek a uniform bound on the quantity \[|f(x)-f_n(x)|
    \leq|f(x)-I(x;1)|+|I_n(x)-I(x;1)|\]
over the compact set \(K\), where \(I(x;1)\) is given by~\eqref{eqn: MC int} for all \(x\in K\).
Since equation~\eqref{eqn: final lim-int rep} allows us to fix \(\alpha,\Omega>0\) such that \[|f(x)-I(x;1)|
    =\Big|f(x)-\int_{K(\Omega)}F_{\alpha,\Omega}(y,w,u)\rho\big(\alpha\langle w,x\rangle+b_{\alpha}(y,w,u)\big)\mathrm{d}y\mathrm{d}w\mathrm{d}u\Big|
    <\sqrt{\frac{\varepsilon}{2\mathrm{vol}(K)}}\]
holds for every \(x\in K\) simultaneously, the result would follow if we show that, with high probability, \hfill \\ \(|I_n(x)-I(x;1)|<\sqrt{\varepsilon/2\mathrm{vol}(K)}\) uniformly for all \(x\in K\). Indeed, this would yield
\[\int_K|f(x)-f_n(x)|^2\mathrm{d}x
    \leq2\int_K|f(x)-I(x;1)|^2\mathrm{d}x+2\int_K|I_n(x)-I(x;1)|^2\mathrm{d}x
    <\varepsilon\]
with high probability.
To this end, for \(\delta>0\) let \(\mathcal{C}(\delta,K)\subset K\) denote a minimal \(\delta\)-net for \(K\), with cardinality \(\mathcal{N}(\delta,K)\).
Now, fix \(x\in K\) and consider the inequality

\begin{linenomath}
\begin{align}\label{eqn: non asmpt bound}
    |I_n(x)-I(x;1)|
    &\leq\underbrace{|I_n(x)-I_n(z)|}_{(*)}+\underbrace{|I_n(z)-I(z;1)|}_{(**)}+\underbrace{|I(x;1)-I(z;1)|}_{(***)},
\end{align}
\end{linenomath}
where \(z\in\mathcal{C}(\delta,K)\) is such that \(\Vert x-z\Vert_2<\delta\).
We will obtain the desired bound on~\eqref{eqn: non asmpt bound} by bounding each of the terms \((*)\), \((**)\), and \((***)\) separately.

First, we consider the term \((*)\).
Recalling the definition of \(I_n\), observe that we have
\begin{linenomath}
\begin{align*}
    (*)
    &=\frac{\mathrm{vol}(K(\Omega))}{n}
    \Big|\sum_{k=1}^nF_{\alpha,\Omega}(y_k,w_k,u_k)\Big(\rho\big(\alpha\langle w_k,x\rangle+b_\alpha(y_k, w_k, u_k)\big)\\
    &\qquad\qquad\qquad\qquad\qquad\qquad\qquad\qquad\qquad-\rho\big(\alpha\langle w_k,z\rangle+b_\alpha(y_k,w_k,u_k)\big)\Big)\Big|\\
    &\leq\frac{\alpha M\mathrm{vol}(K(\Omega))}{(2\pi)^N n}\sum_{k=1}^n\big|\rho\big(\alpha\langle w_k,x\rangle+b_\alpha(y_k, w_k, u_k)\big)-\rho\big(\alpha\langle w_k,z\rangle+b_\alpha(y_k, w_k, u_k)\big)\Big|\\
    &\leq\alpha M(2\pi)^{-N}\mathrm{vol}(K(\Omega))R_{\alpha,\Omega}(x,z),
\end{align*}
\end{linenomath}
where \(M:=\sup_{x\in K}|f(x)|\) and we define \[R_{\alpha,\Omega}(x,z)
    :=\sup_{\substack{y\in K\\w\in[-\Omega,\Omega]^N\\u\in[-(L+\frac{1}{2})\pi,(L+\frac{1}{2})\pi]}}\big|\rho\big(\alpha\langle w,x\rangle+b_\alpha(y,w,u)\big)-\rho\big(\alpha\langle w,z\rangle+b_\alpha(y,w,u)\big)\Big|.\]
Now, since \(\rho\) is assumed to be \(\kappa\)-Lipschitz, we have
\begin{linenomath}
\begin{align*}
    &\big|\rho\big(\alpha\langle w,x\rangle+b_\alpha(y,w,u)\big)-\rho\big(\alpha\langle w,z\rangle+b_\alpha(y,w,u)\big)\Big|\\
    &\qquad=\Big|\rho\Big(\alpha\big(\langle w,x-y\rangle-u\big)\Big)-\rho\Big(\alpha\big(\langle w,z-y\rangle-u\big)\Big)\Big|
    \leq\kappa\alpha\big|\langle w,x-z\rangle\big|
\end{align*}
\end{linenomath}
for any \(y\in K\), \(w\in[-\Omega,\Omega]^N\), and \OPtext{$u\in[-(L+\tfrac{1}{2})\pi,(L+\tfrac{1}{2})\pi].$}
Hence, an application of the Cauchy--Schwarz inequality yields \(R_{\alpha,\Omega}(x,z)\leq\kappa\alpha\Omega\delta\sqrt{N}\) for all \(x\in K\), from which it follows that
\begin{linenomath}
\begin{align}\label{eqn: * bound}
    (*)
    \leq M\sqrt{N}\kappa\delta\alpha^2\Omega(2\pi)^{-N}\mathrm{vol}(K(\Omega))
\end{align}
\end{linenomath}
holds for all \(x\in K\).

Next, we bound \((***)\) using a similar approach. 
Indeed, by the definition of \(I({}\cdot{};1)\) we have
\begin{linenomath}
\begin{align*}
    (***)
    &=\Big|\int_{K(\Omega)}F_{\alpha,\Omega}(y,w,u)\Big(\rho\big(\alpha\langle w,x\rangle+b_\alpha(y,w,u)\big)-\rho\big(\alpha\langle w,z\rangle+b_\alpha(y,w,u)\big)\Big)\mathrm{d}y\mathrm{d}w\mathrm{d}u\Big|\\
    &\leq\frac{\alpha M\lVert\phi\rVert_\infty^N}{(2\pi)^N}\int_{K(\Omega)}\big|\rho\big(\alpha\langle w,x\rangle+b_\alpha(y,w,u)\big)-\rho\big(\alpha\langle w,z\rangle+b_\alpha(y,w,u)\big)\Big|\mathrm{d}y\mathrm{d}w\mathrm{d}u\\
    &\leq\alpha M(2\pi)^{-N}\mathrm{vol}(K(\Omega))R_{\alpha,\Omega}(x,z).
\end{align*}
\end{linenomath}
Using the fact that \(R_{\alpha,\Omega}(x,z)\leq\kappa\alpha\Omega\delta\sqrt{N}\) for al \(x\in K\), it follows that
\begin{linenomath}
\begin{align}\label{eqn: *** bound}
    (***)
    \leq M\sqrt{N}\kappa\delta\alpha^2\Omega(2\pi)^{-N}\mathrm{vol}(K(\Omega))
\end{align}
\end{linenomath}
holds for all \(x\in K\), \OPtext{just like~\eqref{eqn: * bound}.}

Notice that the inequalities~\eqref{eqn: * bound} and~\eqref{eqn: *** bound} are deterministic.
In fact, both can be controlled by choosing an appropriate value for \(\delta\) in the net \(\mathcal{C}(\delta,K)\).
To see this, fix \(\varepsilon'>0\) arbitrarily and recall that \(\mathrm{vol}(K(\Omega))=(2\Omega)^N\pi(2L+1)\mathrm{vol}(K)\).
A simple computation then shows that  \((*)+(***)<\varepsilon'\) whenever
\begin{linenomath}
\begin{align}\label{eqn: net bound}
    \delta
    &<\frac{\varepsilon'}{4\sqrt{N}\kappa\alpha^2M\Omega(\Omega/\pi)^N\mathrm{vol}(K)(\pi+2N\mathrm{rad}(K)\Omega)}\\
    &<\frac{\varepsilon'}{2\sqrt{N}\kappa\alpha^2M\Omega(\Omega/\pi)^N\pi(2L+1)\mathrm{vol}(K)}.\notag
\end{align}
\end{linenomath}

We now bound \((**)\) uniformly for \(x\in K\). 
Unlike \((*)\) and \((***)\), we cannot bound this term deterministically.
In this case, however, we may apply Lemma~\ref{lem: MC concentration ineq} to \[g_z(y,w,u):=F_{\alpha,\Omega}(y,w,u)\rho\big(\alpha\langle w,z\rangle+b_\alpha(y,w,u)\big),\]
for any $z\in\mathcal{C}(\delta,K)$. Indeed, $g_z\in L^2(K(\Omega))$ because $F_{\alpha,\Omega}\in L^2(K(\Omega))$ and $\rho\in L^\infty(\mathbb{R})$. Then Lemma~\ref{lem: MC concentration ineq} yields the tail bound
\begin{linenomath}
\begin{align*}
    \mathbb{P}\big((**)\geq t\big)&=\mathbb{P}\Big(\lvert I_n(g_z,K(\Omega))-I(g_z,K(\Omega))\rvert\geq t\Big)\\
    &\leq3\exp\Big(-\frac{nt}{Bc}\log\big(1+\frac{Bt}{\mathrm{vol}(K(\Omega))I(g_z^2,K(\Omega))}\big)\Big)\\
    &=3\exp\Big(-\frac{nt}{Bc}\log\big(1+\frac{Bt}{\mathrm{vol}(K(\Omega))I(z;2)}\big)\Big)
\end{align*}
\end{linenomath}
for all \(t>0\), where \(c>0\) is a numerical constant and
\begin{linenomath}
\begin{align*}
    B&:=2\alpha M(\Omega/\pi)^N(\pi+2N\mathrm{rad}(K)\Omega)\lVert\rho\rVert_\infty\mathrm{vol}(K)\\
    &\geq\alpha M(\Omega/\pi)^N\pi(2L+1)\lVert\rho\rVert_\infty\mathrm{vol}(K)\\
    &=\alpha M(2\pi)^{-N}\lVert\rho\rVert_\infty\mathrm{vol}(K(\Omega))\\
    &\geq\max_{z\in\mathcal{C}(\delta,K)}\lVert g_z\rVert_\infty\mathrm{vol}(K(\Omega)).
\end{align*}
\end{linenomath}
By taking \[C:=2M\lVert\rho\rVert_\infty\mathrm{vol}(K)
    \quad\text{ and }\quad
    \Sigma:=2C\sqrt{2\mathrm{vol}(K)},
\]
we obtain $B=C\alpha(\Omega/\pi)^N(\pi+2N\mathrm{rad}(K)\Omega)$ and \[\max_{z\in\mathcal{C}(\delta,K)}\mathrm{vol}(K(\Omega))I(z;2)\leq\Big(\alpha M(2\pi)^{-N}\lVert\rho\rVert_\infty\mathrm{vol}(K(\Omega))\Big)^2\leq B^2.
\]
If we choose the number of nodes such that
\begin{linenomath}
\begin{align}\label{eqn: node bound}
    n
    \geq\frac{Bc\log(3\eta^{-1}\mathcal{N}(\delta,K))}{t\log(1+t/B)},
\end{align}
\end{linenomath}
then a union bound yields \((**)<t\) simultaneously for all \(z\in\mathcal{C}(\delta,K)\) with probability at least \(1-\eta\).
Combined with the bounds~\eqref{eqn: * bound} and~\eqref{eqn: *** bound}, it follows from~\eqref{eqn: non asmpt bound} that \[|I_n(x)-I(x;1)|
    <\varepsilon'+t\]
simultaneously for all \(x\in K\) with probability at least \(1-\eta\), provided \(\delta\) and \(n\) satisfy~\eqref{eqn: net bound} and~\eqref{eqn: node bound}, respectively.
Since we require \(|I_n(x)-I(x;1)|<\sqrt{\varepsilon/2\mathrm{vol}(K)}\), the proof is then completed by setting \(\varepsilon'+t=\sqrt{\varepsilon/2\mathrm{vol}(K)}\) and choosing \(\delta\) and \(n\) accordingly.
In particular, it suffices to choose \OPtext{$\varepsilon'=t=\tfrac{1}{2}\sqrt{\varepsilon/2\mathrm{vol}(K)}=C\sqrt{\varepsilon}/\Sigma,$} so that~\eqref{eqn: net bound} and~\eqref{eqn: node bound} become
\begin{linenomath}
\begin{align*}
    \delta
    &<\frac{\sqrt{\varepsilon}}{8\sqrt{2N}\kappa\alpha^2M\Omega(\Omega/\pi)^N\mathrm{vol}^{3/2}(K)(\pi+2N\mathrm{rad}(K)\Omega)},\\
    n
    &\geq\frac{c\Sigma\alpha(\Omega/\pi)^N(\pi+2N\mathrm{rad}(K)\Omega)\log(3\eta^{-1}\mathcal{N}(\delta,K))}{\sqrt{\varepsilon}\log\big(1+\frac{\sqrt{\varepsilon}}{\Sigma\alpha(\Omega/\pi)^N(\pi+2N\mathrm{rad}(K)\Omega)}\big)},
\end{align*}
\end{linenomath}
as desired.
\end{proof}

\begin{remark} The implication of Theorem~\ref{thm: probabilistic Igelnik-Pao} is that, given a desired accuracy level \(\varepsilon>0\), one can construct a RVFL network \(f_n\) that is \(\varepsilon\)-close to \(f\) with high probability, provided the number of nodes \(n\) in the neural network is sufficiently large.
In fact, if we assume that the ambient dimension $N$ is fixed here, then $\delta$ and $n$ depend on the accuracy $\varepsilon$ and probability $\eta$ as \[\delta\lesssim\sqrt{\varepsilon}
    \quad\text{ and }\quad
    n\gtrsim\frac{\log(\eta^{-1}\mathcal{N}(\delta,K))}{\sqrt{\varepsilon}\log\big(1+\sqrt{\varepsilon}\big)}.\]
Using that \(\log(1+x)=x+O(x^2)\) for small values of \(x\), the requirement on the number of nodes behaves like \[n\gtrsim\frac{\log\big(\eta^{-1}\mathcal{N}\big(\sqrt{\varepsilon},K\big)\big)}{\varepsilon}\]
whenever \(\varepsilon\) is sufficiently small.
Using a simple bound on the covering number, this yields a coarse estimate of \(n\gtrsim \varepsilon^{-1}\log(\eta^{-1}/\varepsilon)\).\end{remark}

\begin{remark}
If we instead assume that $N$ is variable, then, under the assumption that  $f$ is H{\"o}lder continuous with exponent $\beta$, one should expect that $n=\omega(N^{2\beta N})$ as $N\to\infty$ (in light of  Remark~\ref{rem: hoelder} and in conjunction with Theorem~\ref{thm: probabilistic Igelnik-Pao} with $\log(1+1/x)\approx 1/x$ for large $x$). In other words, the number of nodes required in the hidden layer is superexponential in the dimension. This dependence of $n$ on $N$ may be improved by means of more refined proof techniques. As for $\alpha,$ if follows from Remark \ref{rem: alpha} that $\alpha=\Theta(1)$ as $N\to\infty$ provided $\int_\mathbb{R}\lvert v\rho(v)\rvert\mathrm{d}v<\infty.$ 
\end{remark}

\begin{remark}
The \(\kappa\)-Lipschitz assumption on the activation function \(\rho\) may likely be removed.
Indeed, since \((***)\) in~\eqref{eqn: non asmpt bound} can be bounded instead by leveraging continuity of the \(L^1\) norm with respect to translation, the only term whose bound depends on the Lipschitz property of \(\rho\) is \((*)\).
However, the randomness in \(I_n\) (that we did not use to obtain the bound~\eqref{eqn: * bound}) may be enough to control \((*)\) in most cases.
Indeed, to bound \((*)\) we require control over quantities of the form
$
    \Big|\rho\Big(\alpha\big(\langle w_k,x-y_k\rangle-u_k\big)\Big)-\rho\Big(\alpha\big(\langle w_k,z-y_k\rangle-u_k\big)\Big)\Big|.
$
For most practical realizations of \(\rho\), this difference will be small with high probability (on the draws of \(y_k,w_k,u_k\)) whenever \(\Vert x-z\Vert_2\) is sufficiently small.
\end{remark}

\longhide{We are interested in obtaining a uniform bound on \(|f_n(x)-I(x;1)|\) for all \(x\in K\), with high probability.
For this, we write
\begin{align*}
    |f_n(x)-I(x;1)|
    &\leq\underbrace{|(I_n(x)-I_n(z))-(I(z;1)-I(x;1))|}_{(*)}+\underbrace{|I_n(z)-I(z;1)|}_{(**)}
\end{align*}
and seek to bound the two terms simultaneously with high probability.
The \((**)\) term argument is unchanged and given below.
For the \((*)\) term, note that \(x\in K\) is fixed and \(z\in\mathcal{C}(\delta,K)\) is fixed such that \(\Vert x-z\Vert_2<\delta\).
With this in mind, we define
\begin{align*}
    S_n(x,z)&:=I_n(x)-I_n(z),\\
    E_n(x,z)&:=I(z;1)-I(x;1)
\end{align*}
for all \(x\in K\), \(z\in\mathcal{C}(\delta,K)\) and observe via Lemma~\ref{lem: MC concentration ineq} that
\begin{align*}
    \mathbb{P}\big(|S_n(x,z)-E_n(x,z)|\geq t\big)
    \leq3\exp\Big(-\frac{nt}{C(x,z)c}\log\big(1+\frac{C(x,z)t}{\mathrm{vol}^2(K(\Omega))\Sigma(x,z)}\big)\Big)
\end{align*}
for any \(t>0\), where \(c>0\) is a numerical constant,
\begin{align*}
    C(x,z)&:=\esssup_{k\in\{1,\ldots,n\}}\Big|\mathrm{vol}(K(\Omega))F_{\alpha,\Omega}\big(y_k,\frac{w_k}{\alpha^N},u_k\big)\Big(\rho\big(\langle w_k,z\rangle+b_k\big)-\rho\big(\langle w_k,x\rangle+b_k\big)\Big)-I(x,z;1)\Big|,\\
    \Sigma(x,z)&:=\frac{I(x,z;2)}{\mathrm{vol}(K(\Omega))}-\frac{I(x,z;1)^2}{\mathrm{vol}^2(K(\Omega))},
\end{align*}
and we have now defined
\begin{align*}
    I(x,z;p)
    :=\int_{K(\Omega)}\bigg(F_{\alpha,\Omega}(y,w,u)\Big(\rho\big(\alpha\langle w,z\rangle+b_{\alpha}(y,w,u)\big)-\rho\big(\alpha\langle w,x\rangle+b_{\alpha}(y,w,u)\big)\Big)\bigg)^p\mathrm{d}y\mathrm{d}w\mathrm{d}u
\end{align*}
for all \(p\in\mathbb{N}\).
Now, since \(\rho\in L^1(\mathbb{R})\), we have
\begin{align*}
    C(x,z)
    \leq2\max\{C_x,C_z\}
    \leq2\max\{\sup_{x\in K}C_x,C\}
    =:\widetilde{C}
    <\infty.
\end{align*}
Moreover, since \(\rho\in L^2(\mathbb{R})\), using~\eqref{eqn: sigma bound} yields
\begin{align*}
    \Sigma(x,z)
    &\leq\frac{\alpha^2M^2}{2^N\mathrm{vol}(K)}\Vert T_{x-z}\rho-\rho\Vert_2^2
\end{align*}
where \(T_{\psi}\colon L^2(\mathbb{R}^N)\rightarrow L^2(\mathbb{R}^N)\) denotes the translation operator which maps \(g(y)\mapsto g(y+\psi)\) for all \(y\in\mathbb{R}^N\), \(g\in L^2(\mathbb{R}^N)\).
To see how we obtain this expression, observe that
\begin{align*}
    \rho\big(\alpha\langle w,z\rangle+b_{\alpha}(y,w,u)\big)
    &=\rho\big(\alpha\langle w,z\rangle-\alpha\langle w,y\rangle-\alpha u\big)\\
    &=\rho\big(\alpha\langle w,x\rangle-\alpha\langle w,y+(x-z)\rangle-\alpha u\big)\\
    &=T_{x-z}\rho\big(\alpha\langle w,z\rangle+b_{\alpha}(y,w,u)\big)
\end{align*}
so that
\begin{align*}
    &\int_K\Big|\rho\big(\alpha\langle w,z\rangle+b_{\alpha}(y,w,u)\big)-\rho\big(\alpha\langle w,x\rangle+b_{\alpha}(y,w,u)\big)\Big|^2\mathrm{d}y\\
    &\qquad\leq\int_{\mathbb{R}^N}\Big|\rho\big(\alpha\langle w,z\rangle+b_{\alpha}(y,w,u)\big)-T_{x-z}\rho\big(\alpha\langle w,z\rangle+b_{\alpha}(y,w,u)\big)\Big|^2\mathrm{d}y
    =\Vert T_{x-z}\rho-\rho\Vert_2^2.
\end{align*}
Now, since the \(L^2\) norm is continuous with respect to translation, for any \(\varepsilon^{\prime}>0\) we may choose our \(\delta>0\) for the net \(\mathcal{C}(\delta,K)\) small enough to ensure that \(\Vert T_{x-z}\rho-\rho\Vert_2<\varepsilon^{\prime}\) for all \(x\in K\).
Hence, if we choose
\begin{align*}
    n
    \geq\frac{\widetilde{C}c\log(6\eta^{-1})}{t\log(1+\frac{2^N\mathrm{vol}(K)\widetilde{C}t}{\alpha^2M^2\mathrm{vol}^2(K(\Omega))(\varepsilon^{\prime})^2})},
\end{align*}
then \((*)<t\) for any \(x\in K\) with probability at least \(1-\eta/2\).
Therefore, combining our results, we have that \((*)+(**)<2t\) with probability at least \(1-\eta\) whenever
\begin{align*}
    n\geq\max\bigg\{\frac{\widetilde{C}c\log(6\eta^{-1})}{t\log(1+\frac{2^N\mathrm{vol}(K)\widetilde{C}t}{\alpha^2M^2\mathrm{vol}^2(K(\Omega))(\varepsilon^{\prime})^2})},\frac{Cc\log(6\eta^{-1}\mathcal{N}(\delta,K))}{t\log(1+\frac{Ct}{\mathrm{vol}^2(K(\Omega))\Sigma})}\bigg\}.
\end{align*}
Since the first term in this max decreases and the second term increases as \(\varepsilon^{\prime}\) decreases, we may fix \(\varepsilon^{\prime}>0\) and \(\delta>0\) in such a way that
\begin{align*}
    \frac{Cc\log(6\eta^{-1}\mathcal{N}(\delta,K))}{t\log(1+\frac{Ct}{\mathrm{vol}^2(K(\Omega))\Sigma})}
    \geq\frac{\widetilde{C}c\log(6\eta^{-1})}{t\log(1+\frac{2^N\mathrm{vol}(K)\widetilde{C}t}{\alpha^2M^2\mathrm{vol}^2(K(\Omega))(\varepsilon^{\prime})^2})},
\end{align*}
from which it follows that \((*)+(**)<2t\) with probability at least \(1-\eta\) whenever
\begin{align*}
    n\geq\frac{Cc\log(6\eta^{-1}\mathcal{N}(\delta,K))}{t\log(1+\frac{Ct}{\mathrm{vol}^2(K(\Omega))\Sigma})}.
\end{align*}
Choosing \(t=\frac{\varepsilon}{4\sqrt{\mathrm{vol}(K)}}\) then gives the final result.
The bound on \(n\) may be expressed
\begin{align*}
    n\geq\frac{4Cc\sqrt{\mathrm{vol}(K)}\log(6\eta^{-1}\mathcal{N}(\delta,K))}{\varepsilon\log(1+\frac{C\varepsilon}{4N^2\Sigma(2\Omega)^{2(N+1)}\mathrm{rad}^2(K)\mathrm{vol}^{5/2}(K)})}.
\end{align*}
THIS DOESN'T QUITE WORK! It works pointwise, but the union bound fails. See below.}

\longhide{
We are interested in obtaining a uniform bound on \(|I_n(x)-I(x;1)|\) for all \(x\in K\), with high probability.
For this, we write
\begin{align*}
    |I_n(x)-I(x;1)|
    &\leq\underbrace{|(I_n(x)-I_n(z))-(I(x;1)-I(z;1))|}_{(*)}+\underbrace{|I_n(z)-I(z;1)|}_{(**)}
\end{align*}
and seek to bound the two terms simultaneously with high probability.
The \((**)\) term argument is unchanged and given below.
For the \((*)\) term, note that \(x\in K\) is fixed and \(z\in\mathcal{C}(\delta,K)\) is fixed such that \(\Vert x-z\Vert_2<\delta\).
With this in mind, we define
\begin{align*}
    S_n(x,z)&:=I_n(x)-I_n(z),\\
    E(x,z)&:=I(x;1)-I(z;1)
\end{align*}
for all \(x\in K\), \(z\in\mathcal{C}(\delta,K)\) and observe via Lemma~\ref{lem: MC concentration ineq} that, for any \(t>0\),
\begin{align*}
    \mathbb{P}\big(|S_n(x,z)-E(x,z)|\geq t\big)
    \leq3\exp\Big(-\frac{nt}{C(x,z)c}\log\big(1+\frac{C(x,z)t}{\mathrm{vol}^2(K(\Omega))\Sigma(x,z)}\big)\Big),
\end{align*}
where \(c>0\) is a numerical constant,
\begin{align*}
    C(x,z)&:=\esssup_{(y,w, u)\in K(\Omega)}\Big|\mathrm{vol}(K(\Omega))F_{\alpha,\Omega}\big(y,w,u\big)\Big(\rho\big(\alpha\langle w,x\rangle+b_{\alpha}(y,w,u)\big)-\rho\big(\alpha\langle w,z\rangle+b_{\alpha}(y,w,u)\big)\Big)\\
    & -I(x,z;1)\Big|,\\
    \Sigma(x,z)&:=\frac{I(x,z;2)}{\mathrm{vol}(K(\Omega))}-\frac{I(x,z;1)^2}{\mathrm{vol}^2(K(\Omega))},
\end{align*}
and we have now defined
\begin{align*}
    I(x,z;p)
    :=\int_{K(\Omega)}\bigg(F_{\alpha,\Omega}(y,w,u)\Big(\rho\big(\alpha\langle w,x\rangle+b_{\alpha}(y,w,u)\big)-\rho\big(\alpha\langle w,z\rangle+b_{\alpha}(y,w,u)\big)\Big)\bigg)^p\mathrm{d}y\mathrm{d}w\mathrm{d}u
\end{align*}
for all \(p\in\mathbb{N}\).
Now, since \(\rho\in L^1(\mathbb{R})\), we have
\begin{align*}
    C(x,z)
    \leq2\max\{C_x,C_z\}
    \leq2\max\{\sup_{x\in K}C_x,C\}
    =:\widetilde{C}
    <\infty.
\end{align*}
Let us define $R_z(y) = \rho\big(\alpha\langle w,z\rangle+b_{\alpha}(y,w,u)\big)=\rho\big(\alpha\langle w,z-y\rangle-\alpha u\big)$. Then we have that $$R_x(y)=\rho\big(\alpha\langle w,z-(y+(z-x))\rangle-\alpha u\big) = T_{z-x}R_z(y),$$ where \(T_{\psi}\colon L^2(\mathbb{R}^N)\rightarrow L^2(\mathbb{R}^N)\) denotes the translation operator which maps \(g(y)\mapsto g(y+\psi)\) for all \(y\in\mathbb{R}^N\), \(g\in L^2(\mathbb{R}^N)\). Moreover, since \(\rho\in L^2(\mathbb{R})\), we have that $R_z\in L^2(\mathbb{R})$ for all $z\in \mathcal{C}(\delta,K)$. Thus, 
using~\eqref{eqn: sigma bound} yields
\begin{align*}
    \Sigma(x,z)
    &\leq\frac{\alpha^2M^2}{2^N\mathrm{vol}(K)}\Vert T_{z-x}R_z-R_z\Vert_2^2
\end{align*}
Now, since the \(L^2\) norm is continuous with respect to translation, for any \(\varepsilon^{\prime}>0\) we may choose our \(\delta>0\) for the net \(\mathcal{C}(\delta,K)\) small enough to ensure that \(\Vert T_{z-x}R_z-R_z\Vert_2<\varepsilon^{\prime}\) for all \(z\in \mathcal{C}(\delta,K)\).
BUT: how do we then do the union bound to obtain uniform for all $x\in K$ probability bound?
}

\longhide{moreover, by continuity of \(f\) we may choose our \(\delta>0\) such that
\begin{align*}
    |f(x)-f(z)|<\frac{\varepsilon}{16\alpha(2\Omega)^{N+1}\mathrm{vol}^{3/2}(K)\Vert\rho\Vert_1},
\end{align*}
from which we obtain
\begin{align*}
    (*)+(***)<\frac{\varepsilon}{4\sqrt{\mathrm{vol}(K)}}.
\end{align*}}

\longhide{if we choose \(t=\varepsilon/4\sqrt{\mathrm{vol}(K)}\) and set
\begin{align*}
    \eta=3\mathcal{N}(\delta,K)\exp\Big(-\frac{n\varepsilon}{4Cc\sqrt{\mathrm{vol}(K)}}\log\big(1+\frac{C\varepsilon}{4(2\Omega)^{2(N+1)}\mathrm{vol}^{5/2}(K)\Sigma}\big)\Big),
\end{align*}
then a union bound yields
\begin{align*}
    (**)
    <\frac{\varepsilon}{4\sqrt{\mathrm{vol}(K)}}
\end{align*}
simultaneously for all \(y\in\mathcal{C}(\delta,K)\) with probability at least \(1-\eta\) whenever
\begin{align*}
    n\geq\frac{4Cc\sqrt{\mathrm{vol}(K)}\log(3\eta^{-1}\mathcal{N}(\delta,K))}{\varepsilon\log(1+\frac{C\varepsilon}{4(2\Omega)^{2(N+1)}\mathrm{vol}^{5/2}(K)\Sigma})},
\end{align*}
that is,
\begin{align*}
    n\gtrsim\frac{\sqrt{\mathrm{vol}(K)}\log(\eta^{-1}\mathcal{N}(\delta,K))}{\varepsilon\log(1+\frac{\varepsilon}{\Omega^N\mathrm{vol}^{5/2}(K)})}.
\end{align*}
With this choice of \(n\), we conclude that
\begin{align*}
    |I_n(x)-I(x;1)|
    \leq(*)+(**)+(***)
    <\frac{\varepsilon}{2\sqrt{\mathrm{vol}(K)}}
\end{align*}
holds uniformly for all \(x\in K\) with probability at least \(1-\eta\).}

\subsection{Results on submanifolds of Euclidean space}\label{sec: manifold results}

The constructions of RVFL networks presented in Theorems~\ref{thm: IP95} and~\ref{thm: probabilistic Igelnik-Pao} depend heavily on the dimension of the ambient space \(\mathbb{R}^N\).
Indeed, the random variables used to construct the input-to-hidden layer weights and biases for these neural networks are \(N\)-dimensional objects; moreover, it follows from \eqref{eqn: net bound} and \eqref{eqn: node bound} that the lower bound on the number $n$ of nodes in the hidden layer depends \OPtext{superexponentially} on the ambient dimension $N$.
If the ambient dimension is small, these dependencies do not present much of a problem.
However, many modern applications require the ambient dimension to be large.
Fortunately, a common assumption in practice is that signals of interest have (e.g., manifold) structure that effectively reduces their complexity.
Good theoretical results and algorithms in a number of settings typically depend on this induced smaller dimension rather than the ambient dimension.
For this reason, it is desirable to obtain approximation results for RVFL networks that leverage the underlying structure of the signal class of interest, namely, the domain of \(f\in C_c(\mathbb{R}^N)\).

One way to introduce lower-dimensional structure in the context of RVFL networks is to assume that \(\mathrm{supp}(f)\) lies on a subspace of \(\mathbb{R}^N\).
More generally, and motivated by applications, we may consider the case where \(\mathrm{supp}(f)\) is actually a submanifold of  \(\mathbb{R}^N\).
To this end, for the remainder of this section, we assume \(\mathcal{M}\subset\mathbb{R}^N\) to be a smooth, compact  \(d\)-dimensional manifold and consider the problem of approximating functions \(f\in C(\mathcal{M})\) using RVFL networks.
As we are going to see, RVFL networks in this setting yield theoretical guarantees that replace the dependencies of Theorems~\ref{thm: IP95} and~\ref{thm: probabilistic Igelnik-Pao} on the ambient dimension \(N\) with dependencies on the manifold dimension \(d\).
\OPtext{%
Indeed, one should expect that the random variables \(\{w_k\}_{k=1}^n\), \(\{b_k\}_{k=1}^n\) are essentially \(d\)-dimensional objects (rather than \(N\)-dimensional) and that the lower bound on the number of network nodes in Theorem~\ref{thm: probabilistic Igelnik-Pao} scales as a (superexponential) function of $d$ rather than $N$.
}%

\subsubsection{Adapting RVFL networks to \(d\)-manifolds}\label{sec: man construction}

As in Section~\ref{sec: manifold basics}, let \(\{(U_j,\phi_j)\}_{j\in J}\) be an atlas for the smooth, compact \(d\)-dimensional manifold \(\mathcal{M}\subset\mathbb{R}^N\) with the corresponding compactly supported partition of unity \(\{\eta_j\}_{j\in J}\).
Since \(\mathcal{M}\) is compact, we assume without loss of generality that \(|J|<\infty\). Indeed, if we additionally assume that $\mathcal{M}$ satisfies the property that there exists an $r>0$ such that, for each \(x\in\mathcal{M}\), \(\mathcal{M}\cap B_2^N(x,r)\) is diffeomorphic to an \(\ell_2\) ball in \(\mathbb{R}^d\) with diffeomorphism close to the identity. Then one can choose an atlas \(\{(U_j,\phi_j)\}_{j\in J}\) with \(|J|\lesssim2^d T_d \mathrm{vol}(\mathcal{M})r^{-d}\) by intersecting \(\mathcal{M}\) with \(\ell_2\) balls in \(\mathbb{R}^N\) of radii \(r/2\)~\citep{shaham2018provable}. Here $T_d$ is the so-called thickness of the covering and there exist coverings such that $T_d\lesssim d\log(d)$.

Now, for \(f\in C(\mathcal{M})\), Lemma~\ref{lem: manifold funct rep} implies that
\begin{linenomath}
\begin{align}\label{eqn: manifold funct rep}
    f(x)=\sum_{\{j\in J\colon x\in U_j\}}(\hat{f}_j\circ\phi_j)(x)
\end{align}
\end{linenomath}
for all \(x\in\mathcal{M}\), where \[\hat{f}_j(z):=\begin{cases}
    f(\phi_j^{-1}(z))\,\eta_j(\phi_j^{-1}(z))\quad&z\in\phi_j(U_j)\\
    0&\text{otherwise}.
    \end{cases}\]

As we will see, the fact that \(\mathcal{M}\) is smooth and compact implies \(\hat{f}_j\in C_c(\mathbb{R}^d)\) for each \(j\in J\), and so we may approximate each \(\hat{f}_j\) using RVFL networks on \(\mathbb{R}^d\) as in Theorems~\ref{thm: IP95} and~\ref{thm: probabilistic Igelnik-Pao}.
In this way, it is reasonable to expect that \(f\) can be approximated on \(\mathcal{M}\) using a linear combination of these low-dimensional RVFL networks.
More precisely, we propose approximating \(f\) on \(\mathcal{M}\) via the following process:
\begin{enumerate}
    \item
    For each \(j\in J\), approximate \(\hat{f}_j\) uniformly on \(\phi_j(U_j)\subset\mathbb{R}^d\) using a RVFL network \(\tilde{f}_{n_j}\) as in Theorems~\ref{thm: IP95} and~\ref{thm: probabilistic Igelnik-Pao};
    \item
    Approximate \(f\) uniformly on \(\mathcal{M}\) by summing these RVFL networks over \(J\), i.e.,
    \begin{align*}
        f(x)\approx\sum_{\{j\in J\colon x\in U_j\}}(\tilde{f}_{n_j}\circ\phi_j)(x)
    \end{align*}
    for all \(x\in\mathcal{M}\).
\end{enumerate}

\longhide{In order to prove approximation results like Theorems~\ref{thm: IP95} and~\ref{thm: probabilistic Igelnik-Pao} in this setting, we must first verify a few minor technical points.
To this end, define
\begin{align*}
    h_w(y):=\prod_{k=1}^\mathrm{d}w(k)\rho\big(w(k)y(k)\big)
\end{align*}
for all \(y,w\in\mathbb{R}^d\), as in~\eqref{eqn: window transform}.
Since it is assumed that \(\mathrm{supp}(\eta_j)\subset U_j\) is compact for each \(j\in J\), it follows that \(\phi_j(\mathrm{supp}(\eta_j))\) is a compact subset of \(\mathbb{R}^d\).
Moreover, because \(\hat{f}_j(z)\neq0\) if and only if \(z\in\phi_j(U_j)\) and \(\phi_j^{-1}(z)\in\mathrm{supp}(\eta_j)\subset U_j\), we have that \(\hat{f}_j=\hat{f}_j|_{\phi_j(\mathrm{supp}(\eta_j)}\) is supported on a compact set.
Hence, \(\hat{f}_j\in C_c(\mathbb{R}^d)\) for each \(j\in J\), and so we may apply Lemma~\ref{lem: Step 1} to obtain
\begin{align}\label{eqn: manifold lim-int rep 1}
    \hat{f}_j(z)
    =\lim_{\Omega_j\rightarrow\infty}\frac{1}{\Omega_j^d}\int_{[0,\Omega_j]^d}(\hat{f}_j*h_w)(z)\mathrm{d}w
\end{align}
for all \(z\in\mathbb{R}^d\).
With the representation~\eqref{eqn: manifold lim-int rep 1} in hand, applying Lemma~\ref{lem: step 2} yields
\begin{align}\label{eqn: manifold final lim-int rep}
    \hat{f}_j(z)
    =\lim_{\Omega_j\rightarrow\infty}\lim_{\alpha_j\rightarrow\infty}
    \int_{K(\Omega_j)}
    F_{\alpha_j,\Omega_j}(y,w,u)\rho\big(\alpha_j\langle w,z\rangle+b_{\alpha_j}(y,w,u)\big)\mathrm{d}y\mathrm{d}w\mathrm{d}u
\end{align}
for all \(z\in\phi_j(U_j)\), where
\begin{align*}
    K(\Omega_j):=\phi_j(U_j)\times[-\Omega_j,\Omega_j]^d\times[-\frac{\pi}{2}(2L_j+1),\frac{\pi}{2}(2L_j+1)]
\end{align*}
for \(L_j:=\lceil\frac{2d}{\pi}\mathrm{rad}(\phi_j(U_j))\Omega_j-\frac{1}{2}\rceil\), and \(F_{\alpha_j,\Omega_j},b_{\alpha_j}\) are defined as in~\eqref{eqn: F-b}; that is,
\begin{align*}
    F_{\alpha_j,\Omega_j}(y,w,u)
    &:=\frac{\alpha_j}{\Omega_j^d}\Big|\prod_{k=1}^\mathrm{d}w(k)\Big|\hat{f}_j(y)\cos_{\Omega_j}(u),\\
    b_{\alpha_j}&:=-\alpha_j(\langle w,y\rangle+u).
\end{align*}
Hence, for each \(j\in J\) we obtain a Monte Carlo approximation via~\eqref{eqn: I-P RVFLs}:
\begin{align*}
    \hat{f}_j(z)\approx\tilde{f}_{n_j}(z)
    :=\sum_{k=1}^{n_j}v_k^{(j)}\rho\big(\langle w_k^{(j)},z\rangle+b_k^{(j)}\big)
\end{align*}
for all \(z\in\phi_j(U_j)\), where \(\{y_k^{(j)}\}_{k=1}^n\), \(\{w_k^{(j)}\}_{k=1}^n\), and \(\{u_k^{(j)}\}_{k=1}^n\) are independent samples drawn uniformly from \(\phi_j(U_j)\), \([-\alpha_j\Omega_j,\alpha_j\Omega_j]^d\), and \([-\frac{\pi}{2}(2L_j+1),\frac{\pi}{2}(2L_j+1)]\), respectively, and \(b_k^{(j)}:=-(\langle w_k^{(j)},y_k^{(j)}\rangle+\alpha_ju_k^{(j)})\) for each \(k=1,\ldots,n_j\).
In particular, by Lemma~\ref{lem: part 3} we have
\begin{align*}
    \lim_{n_j\rightarrow\infty}\mathbb{E}\int_{\phi_j(U_j)}
    \left|\int_{K(\Omega_j)}
    F_{\alpha_j,\Omega_j}(y,w,u)\rho\big(\alpha_j\langle w,z\rangle+b_{\alpha_j}(y,w,u)\big)\mathrm{d}y\mathrm{d}w\mathrm{d}u
    -\tilde{f}_{n_j}(z)\right|^2\mathrm{d}z
    =0
\end{align*}
for each \(j\in J\), with convergence rate \(O(1/n_j)\).
We now approximate \(f\) on \(\mathcal{M}\) by the sum
\begin{align}\label{eqn: manifold RVFL}
    f(x)
    \approx\sum_{\{j\in J\colon x\in U_j\}}(\tilde{f}_{n_j}\circ\phi_j)(x)
    =\sum_{\{j\in J\colon x\in U_j\}}\sum_{k=1}^{n_j}v_k^{(j)}\rho\big(\langle w_k^{(j)},\phi_j(x)\rangle+b_k^{(j)}\big).
\end{align}
}

\subsubsection{Main results on \(d\)-manifolds}\label{sec: manifold IP95 proof}

We now prove approximation results for the manifold RVFL network architecture described in Section~\ref{sec: man construction}.
For notational clarity, from here onward we use \(\lim_{\{n_j\}_{j\in J}\rightarrow\infty}\) to denote the limit as each \(n_j\) tends to infinity simultaneously.
The first theorem that we prove is an asymptotic approximation result for continuous functions on manifolds using the RVFL network construction presented in Section~\ref{sec: man construction}.
This theorem is the manifold-equivalent of Theorem~\ref{thm: IP95}.

\begin{theorem}\label{thm: manifold Igelnik-Pao}
Let \(\mathcal{M}\subset\mathbb{R}^N\) be a smooth, compact \(d\)-dimensional manifold with finite atlas \(\{(U_j,\phi_j)\}_{j\in J}\) and \(f\in C(\mathcal{M})\).
Fix any activation function \OPtext{$\rho\in L^1(\mathbb{R})\cap L^\infty(\mathbb{R})$} with \(\int_{\mathbb{R}}\rho(z)\mathrm{d}z=1\).
For any \(\varepsilon>0\), there exist constants \(\alpha_j,\Omega_j>0\)  for each \(j\in J\) such that the following holds. 
If, for each \(j\in J\) and for $k\in  \mathbb{N}$, the random variables
\begin{linenomath}
\begin{align*}
    w_k^{(j)}&\sim \mathrm{Unif}([-\alpha_j\Omega_j,\alpha_j\Omega_j]^d);\\
    y_k^{(j)}&\sim \mathrm{Unif}(\phi_j(U_j));\\
    u_k^{(j)}&\sim \mathrm{Unif}([-\tfrac{\pi}{2}(2L_j+1),\tfrac{\pi}{2}(2L_j+1)]),
    \quad\text{where \(L_j:=\lceil\tfrac{2d}{\pi}\mathrm{rad}(\phi_j(U_j))\Omega_j-\tfrac{1}{2}\rceil\)},
\end{align*}
\end{linenomath}
are independently drawn from their associated distributions, and \[b_k^{(j)}:=-\langle w_k^{(j)},y_k^{(j)}\rangle-\alpha_ju_k^{(j)},\] then there exist hidden-to-output layer weights \(\{v_k^{(j)}\}_{k=1}^{n_j}\subset\mathbb{R}\) such that the sequences of RVFL networks \(\{\tilde{f}_{n_j}\}_{n_j=1}^{\infty}\) defined by \[\tilde{f}_{n_j}(z):=\sum_{k=1}^{n_j}v_k^{(j)}\rho\big(\langle w_k^{(j)},z\rangle+b_k^{(j)}\big),\quad\text{ for \(z\in\phi_j(U_j)\)}\]
satisfy \[\mathbb{E}\int_{\mathcal{M}}\bigg|f(x)\quad-\sum_{\{j\in J\colon x\in U_j\}}(\tilde{f}_{n_j}\circ\phi_j)(x)\bigg|^2\mathrm{d}x\leq\varepsilon+O(1/\min_{j\in J}n_j)\]
as $\{n_j\}_{j\in J}\rightarrow\infty.$
\end{theorem}


\begin{proof}
We wish to show that there exist sequences of RVFL networks \(\{\tilde{f}_{n_j}\}_{n_j=1}^{\infty}\) defined on \(\phi_j(U_j)\) for each \(j\in J\) which together satisfy the asymptotic error bound \[\mathbb{E}\int_{\mathcal{M}}\bigg|f(x)\quad-\sum_{\{j\in J\colon x\in U_j\}}(\tilde{f}_{n_j}\circ\phi_j)(x)\bigg|^2\mathrm{d}x\leq\varepsilon+O(1/\min_{j\in J}n_j)\] as $\{n_j\}_{j\in J}\rightarrow\infty.$
We will do so by leveraging the result of Theorem~\ref{thm: IP95} on each \(\phi_j(U_j)\subset\mathbb{R}^d\).

To begin, recall that we may apply the representation~\eqref{eqn: manifold funct rep} for \(f\) on each chart \((U_j,\phi_j)\);
the RVFL networks \(\tilde{f}_{n_j}\) we seek are approximations of the functions \(\hat{f}_j\) in this expansion.
Now, as \(\mathrm{supp}(\eta_j)\subset U_j\) is compact for each \(j\in J\), it follows that each set \(\phi_j(\mathrm{supp}(\eta_j))\) is a compact subset of \(\mathbb{R}^d\).
Moreover, because \(\hat{f}_j(z)\neq0\) if and only if \(z\in\phi_j(U_j)\) and \(\phi_j^{-1}(z)\in\mathrm{supp}(\eta_j)\subset U_j\), we have that \(\hat{f}_j=\hat{f}_j|_{\phi_j(\mathrm{supp}(\eta_j)}\) is supported on a compact set.
Hence, \(\hat{f}_j\in C_c(\mathbb{R}^d)\) for each \(j\in J\), and so we may apply Lemma~\ref{lem: step 2} to obtain the uniform limit representation~\eqref{eqn: final lim-int rep} on \(\phi_j(U_j)\), that is, \[\hat{f}_j(z)
    =\lim_{\Omega_j\rightarrow\infty}\lim_{\alpha_j\rightarrow\infty}
    \int_{K(\Omega_j)}
    F_{\alpha_j,\Omega_j}(y,w,u)\rho\big(\alpha_j\langle w,z\rangle+b_{\alpha_j}(y,w,u)\big)\mathrm{d}y\mathrm{d}w\mathrm{d}u,\]
where we define \[K(\Omega_j):=\phi_j(U_j)\times[-\Omega_j,\Omega_j]^d\times[-\tfrac{\pi}{2}(2L_j+1),\tfrac{\pi}{2}(2L_j+1)].\]
In this way,
the asymptotic error bound that is the final result of Theorem~\ref{thm: IP95}, namely
\begin{linenomath}
\begin{align}\label{eqn: man asympt bound}
    \mathbb{E}\int_{\phi_j(U_j)}\big|\hat{f}_j(z)-\tilde{f}_{n_j}(z)\big|^2\mathrm{d}z
    \leq\varepsilon_j+O(1/n_j)\OPtext{,}
\end{align}
\end{linenomath}
\OPtext{holds.} With these results in hand, we may now continue with the main body of the proof.

Since the representation~\eqref{eqn: manifold funct rep} for \(f\) on each chart \((U_j,\phi_j)\) yields \[\bigg|f(x)\quad-\sum_{\{j\in J\colon x\in U_j\}}(\tilde{f}_{n_j}\circ\phi_j)(x)\bigg|
    \leq\sum_{\{j\in J\colon x\in U_j\}}\Big|(\hat{f}_j\circ\phi_j)(x)-(\tilde{f}_{n_j}\circ\phi_j)(x)\Big|\]
for all \(x\in\mathcal{M}\), \OPtext{Jensen's inequality allows us to bound} the mean square error of our RVFL approximation
by
\begin{linenomath}
\begin{align}\label{eqn: manifold bound 1}
    \begin{split}
    &\mathbb{E}\int_{\mathcal{M}}\bigg|f(x)\quad-\sum_{\{j\in J\colon x\in U_j\}}(\tilde{f}_{n_j}\circ\phi_j)(x)\bigg|^2\mathrm{d}x\\
    &\leq\lvert J\rvert\cdot\underbrace{\mathbb{E}\int_{\mathcal{M}}\sum_{\{j\in J\colon x\in U_j\}}\Big|(\hat{f}_j\circ\phi_j)(x)-(\tilde{f}_{n_j}\circ\phi_j)(x)\Big|^2\mathrm{d}x}_{(*)}
    \end{split}
\end{align}
\end{linenomath}
To bound \((*)\), note that the change of variables~\eqref{eqn: man CoVs} implies \[\int_{\mathcal{M}}\sum_{\{j\in J\colon x\in U_j\}}\Big|(\hat{f}_j\circ\phi_j)(x)-(\tilde{f}_{n_j}\circ\phi_j)(x)\Big|^2\mathrm{d}x
    =\sum_{j\in J}\int_{\phi_j(U_j)}\frac{\big|\hat{f}_j(z)-\tilde{f}_{n_j}(z)\big|^2}{|\det(D\phi_j(\phi_j^{-1}(z)))|}\mathrm{d}z\]
for each \(j\in J\).
Defining \(\beta_j:=\inf_{y\in U_j}|\det(D\phi_j(y))|\), which is necessarily bounded away from zero for each \(j\in J\) by compactness of \(\mathcal{M}\), 
we therefore have\[
    (*)
    \leq\sum_{j\in J}\beta_j^{-1}\mathbb{E}\int_{\phi_j(U_j)}\big|\hat{f}_j(z)-\tilde{f}_{n_j}(z)\big|^2\mathrm{d}z.\]
Hence, applying~\eqref{eqn: man asympt bound} for each \(j\in J\) yields
\begin{linenomath}
\begin{align}\label{eqn: manifold intermediate bound 1}
   (*)
    \leq\sum_{j\in J}\beta_j^{-1}\big(\varepsilon_j+O(1/n_j)\big)=\sum_{j\in J}\frac{\varepsilon_j}{\beta_j}+O(1/\min_{j\in J}n_j)
\end{align}
\end{linenomath}
because $\sum_{j\in J}1/n_j\leq\lvert J\rvert/\min_{j\in J}n_j.$
With the \OPtext{bound}~\eqref{eqn: manifold intermediate bound 1}
in hand,
\eqref{eqn: manifold bound 1} \OPtext{becomes} 
\[\mathbb{E}\int_{\mathcal{M}}\bigg|f(x)-\sum_{\substack{\{j\in J\colon\\ x\in U_j\}}}(\tilde{f}_{n_j}\circ\phi_j)(x)\bigg|^2\mathrm{d}x
    \leq\lvert J\rvert\sum_{j\in J}\frac{\varepsilon_j}{\beta_j}
    +O(1/\min_{j\in J}n_j)\]
as $\{n_j\}_{j\in J}\rightarrow\infty,$
and so the proof is completed by taking each \(\varepsilon_j>0\) in such a way that \[\varepsilon=\lvert J\rvert\sum_{j\in J}\frac{\varepsilon_j}{\beta_j},\]
and choosing \(\alpha_j,\Omega_j>0\) accordingly for each \(j\in J\).
\end{proof}

\begin{remark}
Note that the neural-network architecture obtained in Theorem~\ref{thm: manifold Igelnik-Pao} 
has the following form 
in the case of a generic atlas. To obtain the estimate of $f(x)$, the input $x$ is first ``pre-processed'' by computing $\phi_j(x)$ for each $j\in J$ such that $x\in U_j$, and then put  through the corresponding RVFL network. However, using the Geometric Multi-Resolution Analysis approach from \cite{Allard2011MultiscaleGM} (as we do in Section~\ref{sec: numerics}), one can construct an approximation (in an appropriate sense) of the atlas, with maps $\phi_j$ being linear. In this way, the pre-processing step can be replaced by the layer computing $\phi_j(x)$, followed by the RVFL layer~$f_j$. We refer the reader to Section~\ref{sec: numerics} for the details.
\end{remark}

The biggest takeaway from Theorem~\ref{thm: manifold Igelnik-Pao} is that the same asymptotic mean-square error behavior we saw in the RVFL network architecture of Theorem~\ref{thm: IP95} holds for our RVFL-like construction on manifolds, with the added benefit that the input-to-hidden layer weights and biases are now \(d\)-dimensional random variables rather than \(N\)-dimensional.
Provided the size of the atlas \(|J|\) isn't too large, this significantly reduces the number of random variables that must be generated to produce a uniform approximation of \(f\in C(\mathcal{M})\).

One might expect to see a similar reduction in dimension dependence for the non-asymptotic case if the RVFL network construction of Section~\ref{sec: man construction} is used.
Indeed, our next theorem, which is the manifold-equivalent of Theorem~\ref{thm: probabilistic Igelnik-Pao}, makes this explicit:

\begin{theorem}\label{thm: probabilistic manifold Igelnik-Pao}

Let \(\mathcal{M}\subset\mathbb{R}^N\) be a smooth, compact \(d\)-dimensional manifold with finite atlas \(\{(U_j,\phi_j)\}_{j\in J}\) and \(f\in C(\mathcal{M})\).
Fix any activation function \OPtext{$\rho\in L^1(\mathbb{R})\cap L^\infty(\mathbb{R})$} such that \(\rho\) is \(\kappa\)-Lipschitz on \(\mathbb{R}\) for some \(\kappa>0\) and $\int_\mathbb{R}\rho(z)\mathrm{d}z = 1$.
For any \(\varepsilon>0\), there exist constants \(\alpha_j,\Omega_j>0\)  for each \(j\in J\) such that the following holds. 
Suppose, for each \(j\in J\) and for $k=1,...,n_j$, the random variables
\begin{linenomath}
\begin{align*}
    w_k^{(j)}&\sim \mathrm{Unif}([-\alpha_j\Omega_j,\alpha_j\Omega_j]^d);\\
    y_k^{(j)}&\sim \mathrm{Unif}(\phi_j(U_j));\\
    u_k^{(j)}&\sim \mathrm{Unif}([-\tfrac{\pi}{2}(2L_j+1),\tfrac{\pi}{2}(2L_j+1)]),
    \quad\text{where \(L_j:=\lceil\tfrac{2d}{\pi}\mathrm{rad}(\phi_j(U_j))\Omega_j-\tfrac{1}{2}\rceil\)},
\end{align*}
\end{linenomath}
are independently drawn from their associated distributions, and \[b_k^{(j)}:=-\langle w_k^{(j)},y_k^{(j)}\rangle-\alpha_ju_k^{(j)}.\] Then there exist hidden-to-output layer weights \(\{v_k^{(j)}\}_{k=1}^{n_j}\subset\mathbb{R}\) such that, for any \[0
    <\delta_j
    <\frac{\sqrt{\varepsilon}}{8\lvert J\rvert\sqrt{d\mathrm{vol}(\mathcal{M})}\kappa\alpha_j^2M_j\Omega_j(\Omega_j/\pi)^d\mathrm{vol}(\phi_j(U_j))(\pi+2d\mathrm{rad}(\phi_j(U_j))\Omega)},\]
and \[n_j
    \geq\frac{2c\lvert J\rvert\sqrt{\mathrm{vol}(\mathcal{M})}C^{(j)}\alpha_j(\Omega_j/\pi)^d(\pi+2d\mathrm{rad}(\phi_j(U_j))\Omega_j)\log(3\lvert J\rvert\eta^{-1}\mathcal{N}(\delta_j,\phi_j(U_j)))}{\sqrt{\varepsilon}\log\big(1+\frac{\sqrt{\varepsilon}}{2\lvert J\rvert\sqrt{\mathrm{vol}(\mathcal{M})}C^{(j)} \alpha_j(\Omega_j/\pi)^d(\pi+2d\mathrm{rad}(\phi_j(U_j))\Omega_j)}\big)},\]
where \(M_j:=\sup_{z\in\phi_j(U_j)}|\hat{f}_j(z)|\), \(c>0\) is a numerical constant, and \OPtext{$C^{(j)}:=2M_j\lVert\rho\rVert_\infty\mathrm{vol}(\phi_j(U_j)),$}
the sequences of RVFL networks \(\{\tilde{f}_{n_j}\}_{n_j=1}^{\infty}\) defined by \[\tilde{f}_{n_j}(z):=\sum_{k=1}^{n_j}v_k^{(j)}\rho\big(\langle w_k^{(j)},z\rangle+b_k^{(j)}\big),\quad\text{ for \(z\in\phi_j(U_j)\)}\]
satisfy \[\int_{\mathcal{M}}\bigg|f(x)\quad-\sum_{\{j\in J\colon x\in U_j\}}(\tilde{f}_{n_j}\circ\phi_j)(x)\bigg|^2\mathrm{d}x<\varepsilon\]
with probability at least \(1-\eta\).
\end{theorem}

\begin{proof}
See Section \ref{sec: Manifold proof}.
\end{proof}

As alluded to earlier, an important implication of Theorems~\ref{thm: manifold Igelnik-Pao} and~\ref{thm: probabilistic manifold Igelnik-Pao} is that the random variables \(\{w_k^{(j)}\}_{k=1}^{n_j}\) and \(\{b_k^{(j)}\}_{k=1}^{n_j}\) are \(d\)-dimensional objects for each \(j\in J\).
Moreover, bounds for $\delta_j$ and $n_j$ now have \OPtext{superexponential} dependence on the manifold dimension $d$ instead of the ambient dimension $N$.
Thus, introducing the manifold structure removes the dependencies on the ambient dimension, replacing them instead with the intrinsic dimension of \(\mathcal{M}\) and the complexity of the atlas \(\{(U_j,\phi_j)\}_{j\in J}\).

\begin{remark}
The bounds on the covering radii \(\delta_j\) and hidden layer nodes \(n_j\) needed for each chart in Theorem~\ref{thm: probabilistic manifold Igelnik-Pao} are not optimal.
Indeed, these bounds may be further improved if one uses the local structure of the manifold, through quantities such as its curvature and reach.
In particular, the appearance of \(|J|\) in both bounds may be significantly improved upon if the manifold is locally well-behaved.
\end{remark}

\longhide{
\section{Learning the output weights}

Since the results of Theorems~\ref{thm: IP95}, \ref{thm: probabilistic Igelnik-Pao}, \ref{thm: manifold Igelnik-Pao}, and~\ref{thm: probabilistic manifold Igelnik-Pao} are purely existence results, a reasonable question remains as to how to efficiently and effectively learn the weights \(\{v_k\}_{k=1}^n\) in the RVFL networks \(f_n\).
In this section, we seek a means of doing so given a training set \(\{x_k\}_{k=1}^p\subset K\) for some \(p\in\mathbb{N}\) and associated labels \(\{y(k)\}_{k=1}^p\subset\mathbb{R}\).
With some abuse of notation, we may write
\begin{align*}
    y(j)
    =f(x_j)\approx f_n(x_j)
    =\sum_{k=1}^nv_kg\big(\langle w_k,x_j\rangle+b_k\big)
    =a^Tg\big(W^Tx_j+b\big)
    =:a^Tz_{w,b}(x_j)
\end{align*}
for each \(j\in\{1,\ldots,p\}\), where now \(a\in\mathbb{R}^{n\times 1}\) and \(z_{w,b}(\cdot)\in\mathbb{R}^{n\times 1}\).
Hence, we have
\begin{align*}
    y=f(X)\approx f_n(X)=Z_{w,b}(X)^Ta,
\end{align*}
where \(Z_{w,b}(X)\in\mathbb{R}^{n\times p}\) is the matrix whose \(j\)th column is \(z_{w,b}(x_j)\).
We are now left with the solving of the system \(y=Z_{w,b}(X)^Ta\) for the vector \(a\) (note that such \(a\) exists with high probability by Theorem~\ref{thm: probabilistic Igelnik-Pao}).
In the underparameterized case where \(p\geq n\) and \(\mathrm{rank}(Z_{w,b}(X))=n\) (this will depend on the distribution of the training data), the vector \(a\) is unique and given by the least squares solution
\begin{align*}
    a=\big(Z_{w,b}(X)Z_{w,b}(X)^T\big)^{-1}Z_{w,b}(X)y.
\end{align*}
(It should be remarked that the invertibility of \(Z_{w,b}(X)Z_{w,b}(X)^T\) is also highly dependent on the properties of the activation function \(g\).)
We therefore seek to design our training set so that
\begin{align*}
    \big\Vert I-Z_{w,b}(X)^T\big(Z_{w,b}(X)Z_{w,b}(X)^T\big)^{-1}Z_{w,b}(X)\big\Vert_F
\end{align*}
is small.

The remaining (overparameterized) case is when \(p<n\), where the vector \(a\) is not unique (the matrix \(Z_{w,b}(X)^T\) has a null-space).
In this case, the solution given based on the training data \((X,y)\) may or may not be a good solution in general.
In fact, such a solution is just one of many:
\begin{align*}
    a=Z_{w,b}(X)\big(Z_{w,b}(X)^TZ_{w,b}(X)\big)^{-1}y+\Big(I-Z_{w,b}(X)\big(Z_{w,b}(X)^TZ_{w,b}(X)\big)^{-1}Z_{w,b}(X)^T\Big)\xi
\end{align*}
for arbitrary \(\xi\in\mathbb{R}^n\).
Hence, an interesting problem is the following:
Given an initialization \(a_0\in\mathbb{R}^n\), find the closest point \(a\in\mathbb{R}^n\) satisfying \(y=Z_{w,b}(X)^Ta\).
Formally, we have the optimization problem
\begin{align*}
    &\text{minimize}&&\tfrac{1}{2}\Vert h\Vert_2^2\\
    &\text{subject to}&&Z_{w,b}(X)^Th=y-Z_{w,b}(X)^Ta_0,
\end{align*}
where \(h:=a-a_0\).
Using the method of Lagrange multipliers to solve for \(h\),
the updated weight vector is given  by
\begin{align*}
    a=a_0+Z_{w,b}(X)\big(Z_{w,b}(X)^TZ_{w,b}(X)\big)^{-1}\big(y-Z_{w,b}(X)^Ta_0\big)
\end{align*}
and we seek to design our training set so that
\begin{align*}
    \big\Vert I-Z_{w,b}(X)\big(Z_{w,b}(X)^TZ_{w,b}(X)\big)^{-1}Z_{w,b}(X)^T\big\Vert_F
\end{align*}
is small.

Now, a reasonable question to ask is the following: does gradient descent yield a solution \(a\) consistent with that obtained via simple linear regression?
As such regression (essentially kernel learning from the random features \(\{z_{w,b}(x_j)\}_{j=1}^p\)) can be difficult computationally (computing inverses is hard!), gradient descent could offer a (potentially) more efficient option.
Specifically, we seek to apply gradient descent to find the minimizer(s) of the function \(f(a):=\frac{1}{2}\Vert Z_{w,b}(X)^Ta-y\Vert_2^2\).
For this, the gradient descent iterates are given by
\begin{align*}
    a^{(k+1)}
    &=a^{(k)}-\eta\nabla_af(a^{(k)})\\
    &=a^{(k)}-\eta Z_{w,b}(X)\big(Z_{w,b}(X)^Ta^{(k)}-y\big)
    =\big(I-\eta Z_{w,b}(X)Z_{w,b}(X)^T\big)a^{(k)}+\eta Z_{w,b}(X)y
\end{align*}
for each \(k\in\mathbb{N}\), where the initial iterate \(a^{(0)}\) is given arbitrarily and the step size \(\eta>0\) is chosen for convergence.
By an inductive argument, it can be shown that these iterates behave in the following manner:
\begin{align}\label{eqn: GD iterates}
    a^{(k)}
    =\big(I-\eta Z_{w,b}(X)Z_{w,b}(X)^T\big)^ka^{(0)}+\eta\Big(\sum_{j=1}^k\big(I-\eta Z_{w,b}(X)Z_{w,b}(X)^T\big)^{j-1}\Big)Z_{w,b}(X)y.
\end{align}
We now appeal to the following lemma concerning Neumann series:
\begin{lemma}\label{lem: Neumann series inverse}
Let \(Z_{w,b}(X)\) be a matrix as above.
If \(p\geq n\) and the step size \(\eta>0\) is chosen such that \(\Vert I-\eta Z_{w,b}(X)Z_{w,b}(X)^T\Vert<1\), then
\begin{align*}
    \sum_{j=1}^{\infty}\big(I-\eta Z_{w,b}(X)Z_{w,b}(X)^T\big)^{j-1}
    =\frac{1}{\eta}\big(Z_{w,b}(X)Z_{w,b}(X)^T\big)^{-1}.
\end{align*}
Likewise, if \(p<n\) and the step size \(\eta>0\) is chosen such that \(\Vert I-\eta Z_{w,b}(X)^TZ_{w,b}(X)\Vert<1\), then
\begin{align*}
    \sum_{j=1}^{\infty}\big(I-\eta Z_{w,b}(X)^TZ_{w,b}(X)\big)^{j-1}
    =\frac{1}{\eta}\big(Z_{w,b}(X)^TZ_{w,b}(X)\big)^{-1}.
\end{align*}
\end{lemma}
Applying Lemma~\ref{lem: Neumann series inverse} in the gradient descent iterate~\eqref{eqn: GD iterates}, it follows that (with the appropriate choices of the step size \(\eta\))
\begin{align*}
    \lim_{k\rightarrow\infty}a^{(k)}
    =\big(Z_{w,b}(X)Z_{w,b}(X)^T\big)^{-1}Z_{w,b}(X)y
\end{align*}
when \(p\geq n\) and
\begin{align*}
    \lim_{k\rightarrow\infty}a^{(k)}
    =\mathcal{P}a^{(0)}
    +Z_{w,b}(X)\big(Z_{w,b}(X)^TZ_{w,b}(X)\big)^{-1}y
\end{align*}
when \(p<n\), where \(\mathcal{P}\) is orthogonal projection onto the null space of \(Z_{w,b}(X)^T\).
Thus, if \(a^{(0)}\) is chosen to be zero, as is typically the case (anything not in the null space of \(Z_{w,b}(X)^T\) suffices), gradient descent converges to the least squares solution in  both the underparameterized and overparameterized cases.
}

\subsection{Numerical Simulations}\label{sec: numerics}

In this section, we provide numerical evidence to support the result of Theorem~\ref{thm: probabilistic manifold Igelnik-Pao}.
Let \(\mathcal{M}\subset\mathbb{R}^N\) be a smooth, compact \(d\)-dimensional manifold.
Since having access to an atlas for \(\mathcal{M}\) is not necessarily practical, we assume instead that we have a suitable approximation to \(\mathcal{M}\).
For our purposes, we will use a Geometric Multi-Resolution Analysis (GMRA) approximation of \(\mathcal{M}\) ~(see \cite{Allard2011MultiscaleGM}; and also, e.g., \cite{iwen2018recovery} for a complete definition).

A GMRA approximation of \(\mathcal{M}\) provides a collection \(\{(\mathcal{C}_j,\mathcal{P}_j)\}_{j\in\{1,\ldots J\}}\) of centers \({\mathcal{C}_j=\{c_{j,k}\}_{k=1}^{K_j}\subset\mathbb{R}^N}\) and affine projections \(\mathcal{P}_j=\{P_{j,k}\}_{k=1}^{K_j}\) on \(\mathbb{R}^N\) such that, for each \(j\in\{1,\ldots,J\}\), the pairs \(\{(c_{j,k},P_{j,k})\}_{k=1}^{K_j}\) define \(d\)-dimensional affine spaces that approximate \(\mathcal{M}\) with increasing accuracy in the following sense.
For every \(x\in\mathcal{M}\), there exists \(\widetilde{C}_x>0\) and \(k'\in\{1,\ldots,K_j\}\) such that
\begin{linenomath}
\begin{align}\label{eqn: GMRA accuracy}
    \Vert x-P_{j,k'}x\Vert_2\leq \widetilde{C}_x2^{-j}
\end{align}
\end{linenomath}
holds whenever \(\Vert x-c_{j,k'}\Vert_2\) is sufficiently small.
In this way, a GMRA approximation of \(\mathcal{M}\) essentially provides a collection of approximate tangent spaces to \(\mathcal{M}\).
Hence, a GMRA approximation having fine enough resolution (i.e., large enough \(j\)) is a good substitution for an atlas. In practice, one must often first construct a GMRA from empirical data, assumed to be sampled from appropriate distributions on the manifold. Indeed, this is possible, and yields the so-called empirical GMRA, studied in \cite{maggioni2016multiscale}, where finite-sample error bounds are provided. The main point is that given enough samples on the manifold, one can construct a good GMRA approximation of the manifold. 

Let \(\{(c_{j,k}, P_{j,k})\}_{k=1}^{K_j}\) be a GMRA approximation of \(\mathcal{M}\) for refinement level \(j\).
Since the affine spaces defined by \((c_{j,k},P_{j,k})\) for each \(k\in\{1,\ldots,K_j\}\) are \(d\)-dimensional, we will approximate \(f\) on \(\mathcal{M}\) by projecting it (in an appropriate sense) onto these affine spaces and approximating each projection using an RVFL network on \(\mathbb{R}^d\).
To make this more precise, observe that, since each affine projection acts on \(x\in\mathcal{M}\) as \(P_{j,k}x=c_{j,k}+\Phi_{j,k}(x-c_{j,k})\) for some othogonal projection \(\Phi_{j,k}\colon\mathbb{R}^N\rightarrow\mathbb{R}^N\), for each \(k\in\{1,\ldots K_j\}\) we have \[f(P_{j,k}x)
    =f\big(c_{j,k}+\Phi_{j,k}(x-c_{j,k})\big)
    =f\big((I_N-\Phi_{j,k})c_{j,k}+U_{j,k}D_{j,k}V_{j,k}^Tx\big),\]
where \(\Phi_{j,k}=U_{j,k}D_{j,k}V_{j,k}^T\) is the compact singular value decomposition (SVD) of \(\Phi_{j,k}\) (i.e., only the left and right singular vectors corresponding to nonzero singular values are computed).
In particular, the matrix of right-singular vectors \(V_{j,k}\colon\mathbb{R}^d\rightarrow\mathbb{R}^N\) enables us to define a function \(\hat{f}_{j,k}\colon\mathbb{R}^d\rightarrow\mathbb{R}\), given by
\begin{linenomath}
\begin{align}\label{eqn: GMRA funct approx}
    \hat{f}_{j,k}(z)
    :=f\big((I_N-\Phi_{j,k})c_{j,k}+U_{j,k}D_{j,k}z\big),
    \qquad z\in\mathbb{R}^d,
\end{align}
\end{linenomath}
which satisfies \(\hat{f}_{j,k}(V_{j,k}^Tx)=f(P_{j,k}x)\) for all \(x\in\mathcal{M}\).
By continuity of \(f\) and~\eqref{eqn: GMRA accuracy}, this means that for any \(\varepsilon>0\) there exists \(j\in\mathbb{N}\) such that \(|f(x)-\hat{f}_{j,k}(V_{j,k}^Tx)|<\varepsilon\) for some \(k\in \{1,\ldots,K_j\}\). 
For such \(k\in \{1,\ldots,K_j\}\), we may therefore approximate \(f\) on the affine space associated with \((c_{j,k},P_{j,k})\) by approximating \(\hat{f}_{j,k}\) using a RFVL network \(\tilde{f}_{n_{j,k}}\colon\mathbb{R}^d\rightarrow\mathbb{R}\) of the form

\begin{linenomath}
\begin{align}\label{eqn: GMRA RVFL}
    \tilde{f}_{n_{j,k}}(z):=\sum_{\ell=1}^{n_{j,k}}v_{\ell}^{(j,k)}\rho\big(\langle w_{\ell}^{(j,k)},z\rangle+b_{\ell}^{(j,k)}\big),
\end{align}
\end{linenomath}
where \(\{w_{\ell}^{(j,k)}\}_{\ell=1}^{n_{j,k}}\subset\mathbb{R}^d\) and \(\{b_{\ell}^{(j,k)}\}_{\ell=1}^{n_{j,k}}\subset\mathbb{R}\) are random input-to-hidden layer weights and biases (resp.) and the hidden-to-output layer weights \(\{v_{\ell}^{(j,k)}\}_{\ell=1}^{n_{j,k}}\subset\mathbb{R}\) are learned.
Choosing the activation function \(\rho\) and random input-to-hidden layer weights and biases as in Theorem~\ref{thm: probabilistic manifold Igelnik-Pao} then guarantees that \(|f(P_{j,k}x)-\tilde{f}_{n_{j,k}}(V_{j,k}^Tx)|\) is small with high probability whenever \(n_{j,k}\) is sufficiently large.

In light of the above discussion, we propose the following RVFL network construction for approximating functions \(f\in C(\mathcal{M})\):
Given a GMRA approximation of \(\mathcal{M}\) with sufficiently high resolution \(j\), construct and train RVFL networks of the form~\eqref{eqn: GMRA RVFL} for each \(k\in \{1,\ldots,K_j\}\).
Then, given \(x\in\mathcal{M}\) and \(\varepsilon>0\), choose \(k'\in\{1,\ldots,K_j\}\) such that \[c_{j,k'}\in\argmin_{c_{j,k}\in\mathcal{C}_j}\Vert x-c_{j,k}\Vert_2\]
and evaluate \(\tilde{f}_{n_{j,k'}}(x)\) to approximate \(f(x)\).
We summarize this algorithm in Algorithm~\ref{alg: RVFL main}.
Since the structure of the GMRA approximation implies \(\Vert x-P_{j,k'}x\Vert_2\leq C_x2^{-2j}\) holds for our choice of \linebreak \({k'\in\{1,\ldots,K_j\}}\)~\citep[see][]{iwen2018recovery}, continuity of \(f\) and Lemma~\ref{lem: part 3} imply that, for any $\varepsilon > 0$ and \(j\) large enough, \[|f(x)-\tilde{f}_{n_{j,k'}}(V_{j,k'}^Tx)|
    \leq |f(x)-\hat{f}_{j,k'}(V_{j,k'}^Tx)|+|\hat{f}_{j,k'}(V_{j,k'}^Tx)-\tilde{f}_{n_{j,k'}}(V_{j,k'}^Tx)|
    <\varepsilon\]
holds with high probability, provided \(n_{j,k'}\) satisfies the requirements of Theorem~\ref{thm: probabilistic manifold Igelnik-Pao}.

\begin{algorithm}[ht]
\caption{Approximation Algorithm}\label{alg: RVFL main}
    \begin{algorithmic}
        \LState \textbf{Given:} \(f\in C(\mathcal{M})\); GMRA approximation \(\{(c_{j,k}, P_{j,k})\}_{k=1}^{K_j}\) of \(\mathcal{M}\) at scale \(j\)
        
        \LState \textbf{Output:} \(y^{\sharp}\approx f(x)\) for any \(x\in\mathcal{M}\)
        \vspace{5pt}
        
        \LState \textbf{Step 1:} For each \(k\in \{1,\ldots,K_j\}\), construct and train\footnotemark a RVFL network \(\tilde{f}_{n_{j,k}}\) of the form~\eqref{eqn: GMRA RVFL}
        
        \LState \textbf{Step 2:} For any \(x\in\mathcal{M}\), find \(c_{j,k'}\in\argmin_{c_{j,k}\in\mathcal{C}_j}\Vert x-c_{j,k}\Vert_2\)
        
        \LState \textbf{Step 3:} Set \(y^{\sharp}=\tilde{f}_{n_{j,k'}}(x)\)
    \end{algorithmic}
\end{algorithm}
\footnotetext{The construction and training of RVFL networks is left as a ''black box'' procedure. How to best choose a specific activation function \(\rho(z)\) and train each RVFL network~\eqref{eqn: GMRA RVFL} is outside of the scope of this paper. The reader may, for instance, select from the range of methods available for training neural networks.}


\begin{remark}
In the RVFL network construction proposed above we require that the function \(f\) be defined in a sufficiently large region around the manifold.
Essentially, we need to ensure that \(f\) is continuously defined on the set \(S:=\mathcal{M}\cup\widehat{\mathcal{M}}_j\), where \(\widehat{\mathcal{M}}_j\) is the scale-\(j\) GMRA approximation \[\widehat{\mathcal{M}}_{j} := \{P_{j, k_j(z)}z\; \colon \; \|z\|_2 \leq \mathrm{rad}(\mathcal{M}) \} \cap B_2^N(0,\mathrm{rad}(\mathcal{M})).\]
This ensures that \(f\) can be evaluated on the affine subspaces given by the GMRA. 
\end{remark}

To simulate Algorithm~\ref{alg: RVFL main}, we take \(\mathcal{M}=\mathbb{S}^2\) embedded in \(\mathbb{R}^{20}\) and construct a GMRA up to level \(j_{\max}=15\) using 20,000 data points sampled uniformly from \(\mathcal{M}\).
Given \(j\leq j_{\max}\), we generate RVFL networks \(\hat{f}_{n_{j,k}}\colon\mathbb{R}^2\rightarrow\mathbb{R}\) as in~\eqref{eqn: GMRA RVFL} and train them on \(V_{j,k}^T(B_2^N(c_{j,k},r)\cap T_{j,k})\) using the training pairs \(\{(V_{k,j}^Tx_{\ell},f(P_{j,k}x_{\ell}))\}_{\ell=1}^p\), where \(T_{k,j}\) is the affine space generated by \((c_{j,k},P_{j,k})\).
For simplicity, we fix \(n_{j,k}=n\) to be constant for all \(k\in\{1,\ldots,K_j\}\) and use a single, fixed pair of parameters \(\alpha,\Omega>0\) when constructing all RVFL networks.
We then randomly select a test set of 200 points \(x\in\mathcal{M}\) for use throughout all experiments.
In each experiment (i.e., point in Figure~\ref{fig: RVFL error 2}), we use Algorithm~\ref{alg: RVFL main} to produce an approximation \(y^{\sharp}=\tilde{f}_{n_{j,k'}}(x)\) of \(f(x)\).
Figure~\ref{fig: RVFL error 2} displays the mean relative error in these approximations for varying numbers of nodes \(n\); to construct this plot, \(f\) is taken to be the exponential \(f(x)=\exp(\sum_{k=1}^{N}x(k))\) and \(\rho\) the hyperbolic secant function.
Notice that for small numbers of nodes the RVFL networks are not very good at approximating \(f\), regardless of the choice of \(\alpha,\Omega>0\).
However, the error decays as the number of nodes increases until reaching a floor due to error inherent in the GMRA approximation.
Hence, as suggested by Theorem~\ref{thm: manifold Igelnik-Pao_short}, to achieve a desired error bound of \(\varepsilon>0\), one needs to only choose a GMRA scale \(j\) such that the inherent error in the GMRA (which scales like \(2^{-j}\)) is less than \(\varepsilon\), then adjust the parameters \(\alpha_j\), \(\Omega_j\), and \(n_{j,k}\) accordingly.
\OPtext{%
\begin{remark}
As we just mentioned, the error can only decay so far due to the resolution of the GMRA approximation. However, that is not the only floor in our simulation; indeed, the $\varepsilon$ in Theorem~\ref{thm: manifold Igelnik-Pao_short} is determined by the $\alpha_j$'s and $\Omega_j$'s, which we keep fixed (see the caption of Figure~\ref{fig: RVFL error 2}). Consequently, the stagnating accuracy as $n$ increases, as seen in Figure~\ref{fig: RVFL error 2}, is also predicted by Theorem~\ref{thm: manifold Igelnik-Pao_short}. Since the solid and dashed lines seem to reach the same floor, the floor due to error inherent in the GMRA approximation seems to be the limiting error term for RVFL networks with large numbers of nodes. 
\end{remark}
\begin{remark}
Utilizing random inner weights and biases resulted in us needing to approximate the atlas to the manifold. To this end, knowing the computational complexity of the GMRA approximation would be useful in practice. As it turns out in \cite{gmrarate}, calculating the GMRA approximation has computational complexity $O(C^d N m\log(m)),$ where $m$ is the number of training data points and $C>0$ is a numerical constant.
\end{remark}
}%


\begin{figure}[h]
    \centering
   \includegraphics[scale=0.65]{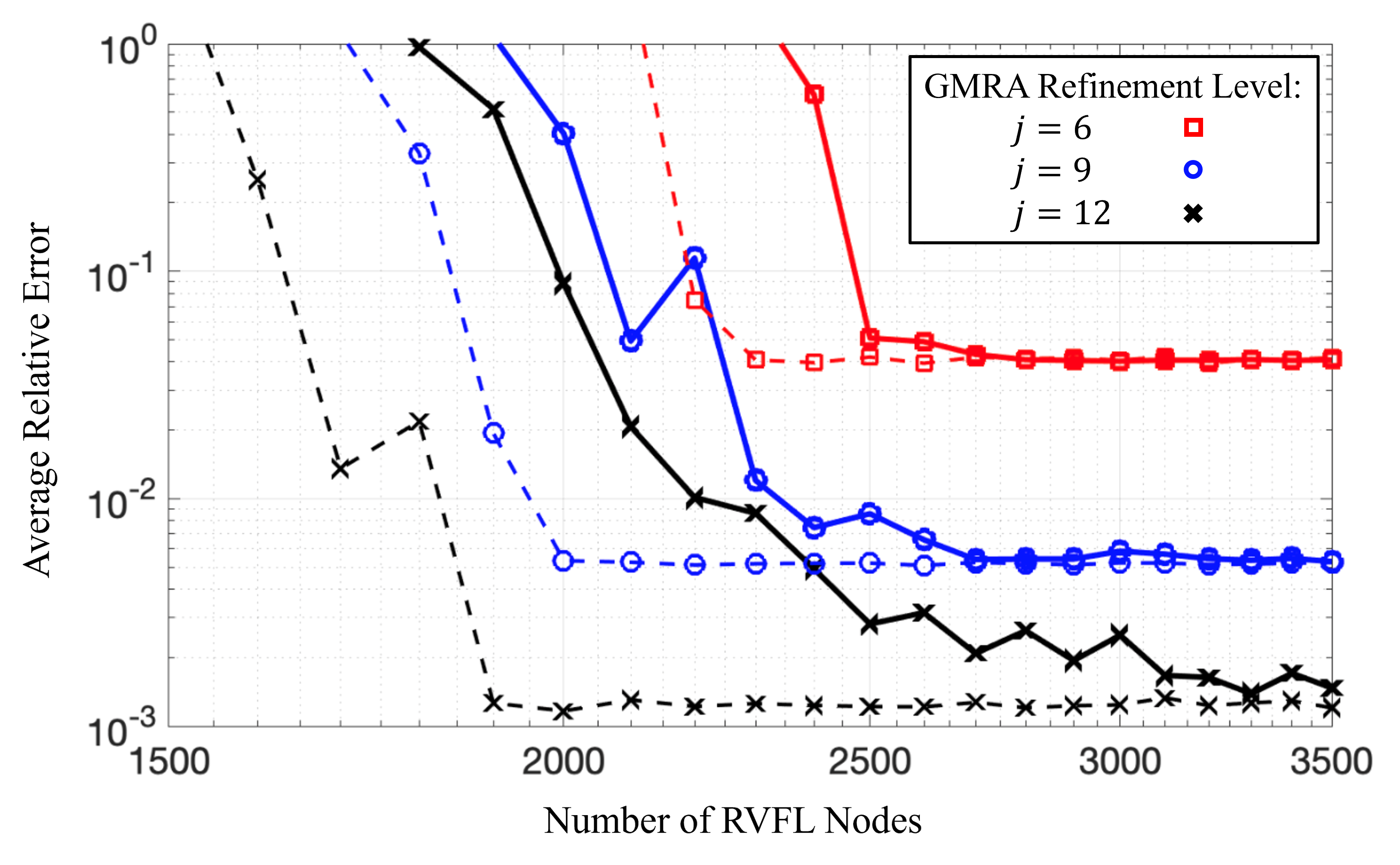}
    \caption{Log-scale plot of average relative error for Algorithm~\ref{alg: RVFL main} as a function of the number of nodes \(n\) in each RVFL network.
    Black (cross), blue (circle), and red (square) lines correspond to GMRA refinement levels \(j=12\), \(j=9\), and \(j=6\) (resp.).
    For each \(j\), we fix \OPtext{$\alpha_j=2$} and vary \OPtext{$\Omega_j=10,15$} (solid and dashed lines, resp.).
   Reconstruction error decays as a function of \(n\) until reaching a floor due to error in the GMRA approximation of \(\mathcal{M}\).
   The code used to obtain these numerical results is available upon direct request sent to the corresponding author.
   }
    \label{fig: RVFL error 2}
\end{figure}

\subsection{Proofs of technical lemmas}\label{sec:app}
\subsubsection{Proof of Lemma \ref{lem: Step 1}}\label{sec: step 1}

Observe that \OPtext{$h_\Omega$} defined in \eqref{eqn: window transform} may be viewed as a multidimensional bump function\OPtext{;} 
indeed, the parameter \OPtext{$\Omega>0$} controls the width of the bump\OPtext{.} 
In particular, if \OPtext{$\Omega$} is allowed to grow very large, then \OPtext{$h_\Omega$} becomes very localized near the origin.
Objects that behave in this way are known in the functional analysis literature as approximate \(\delta\)-functions:
\begin{definition}\label{def: approx delta}
A sequence of functions \(\{\varphi_t\}_{t>0}\subset L^1(\mathbb{R}^N)\) are called \emph{approximate} (or \emph{nascent}) \emph{\(\delta\)-functions} if \[\lim_{t\rightarrow\infty}\int_{\mathbb{R}^N}\varphi_t(x)f(x)\mathrm{d}x
    =f(0)\]
for all \(f\in C_c(\mathbb{R}^N)\).
For such functions, we write \(\delta_0(x)=\lim_{t\rightarrow\infty}\varphi_t(x)\) for all \(x\in\mathbb{R}^N\), where \(\delta_0\) denotes the \(N\)-dimensional Dirac \(\delta\)-function centered at the origin.
\end{definition}

Given \(\varphi\in L^1(\mathbb{R}^N)\) with \(\int_{\mathbb{R}^N}\varphi(x)\mathrm{d}x=1\), one may construct approximate \(\delta\)-functions for \(t>0\) by defining \(\varphi_t(x):=t^N\varphi(tx)\) for all \(x\in\mathbb{R}^N\) \citep{stein1971introduction}.
Such sequences of approximate \(\delta\)-functions are also called \emph{approximate identity sequences}~\citep{rudin1991functional} since they satisfy a particularly nice identity with respect to convolution, namely, \(\lim_{t\rightarrow\infty}\Vert f*\varphi_t-f\Vert_1=0\) for all \(f\in C_c(\mathbb{R}^N)\) \citep[see][Theorem~6.32]{rudin1991functional}.
In fact, such an identity holds much more generally.

\begin{lemma}\cite[Theorem~1.18]{stein1971introduction}\label{lem: approx delta conv}
Let \(\varphi\in L^1(\mathbb{R}^N)\) with \(\int_{\mathbb{R}^N}\varphi(x)\mathrm{d}x=1\) and for \(t>0\) define \(\varphi_t(x):=t^N\varphi(tx)\) for all \(x\in\mathbb{R}^N\).
If \(f\in L^p(\mathbb{R}^N)\) for \(1\leq p<\infty\) (or \(f\in C_0(\mathbb{R}^N)\subset L^{\infty}(\mathbb{R}^N)\) for \(p=\infty\)), then \(\lim_{t\rightarrow\infty}\Vert f*\varphi_t-f\Vert_p=0\).
\end{lemma}

\OPtext{To prove~\eqref{eqn: lim-int rep 2}, it would suffice to have $\lim_{\Omega\to\infty}\lVert f*h_\Omega-f\rVert_{\infty}=0,$ which is really just Lemma~\ref{lem: approx delta conv} in case $p=\infty.$ Nonetheless, we present a proof by mimicking~\cite{stein1971introduction} for completeness. Moreover, we will use a part of proof in Remark~\ref{rem: hoelder} below.}

\begin{lemma}\label{lem: conv equality}
Let \OPtext{$h\in L^1(\mathbb{R}^n)$} with \OPtext{$\int_{\mathbb{R}^N}h(x)\mathrm{d}x=1$} and define \OPtext{$h_\Omega\in L^1(\mathbb{R}^N)$} as in~\eqref{eqn: window transform}
for all \OPtext{$\Omega>0.$}
Then, for all \(f\in C_0(\mathbb{R}^N)\), we have \[\lim_{\Omega\to\infty}\sup_{x\in\mathbb{R}^N}\big\lvert(f*h_\Omega)(x)-f(x)\big\rvert=0.\]
\end{lemma}

\begin{proof}
By symmetry of the convolution operator in its arguments, we have
\begin{linenomath}
\begin{align*}   \sup_{x\in\mathbb{R}^N}\big\lvert(f*h_\Omega)(x)-f(x)\big\rvert
    &=\sup_{x\in\mathbb{R}^N}\Big\lvert\int_{\mathbb{R}^N}f(y)h_\Omega(x-y)\mathrm{d}y-f(x)\Big\rvert\\
    &=\sup_{x\in\mathbb{R}^N}\Big\lvert\int_{\mathbb{R}^N}f(x-y)h_\Omega(y)\mathrm{d}y-f(x)\Big\rvert.
\end{align*}
\end{linenomath}
Since a simple substitution yields \OPtext{$1=\int_{\mathbb{R}^N}h(x)\mathrm{d}x=\int_{\mathbb{R}^N}h_\Omega(x)\mathrm{d}x,$ it follows that}
\begin{linenomath}
\begin{align*}
    \sup_{x\in\mathbb{R}^N}\big\lvert(f*h_\Omega)(x)-f(x)\big\rvert
    &=\sup_{x\in\mathbb{R}^N}\Big\lvert\int_{\mathbb{R}^N}\big(f(x-y)-f(x)\big)h_\Omega(y)\mathrm{d}y\Big\rvert\\
    &\leq\int_{\mathbb{R}^N}\lvert h_\Omega(y)\rvert\sup_{x\in\mathbb{R}^N}\big\lvert f(x)-f(x-y)\big\rvert\mathrm{d}y.
\end{align*}
\end{linenomath}
Finally, expanding the function \OPtext{$h_\Omega,$} we obtain
\begin{linenomath}
\begin{align*}
\sup_{x\in\mathbb{R}^N}\big\lvert(f*h_\Omega)(x)-f(x)\big\rvert
    &\leq\int_{\mathbb{R}^N}\Big(\Omega^N\lvert h(\Omega y)\rvert\Big)\sup_{x\in\mathbb{R}^N}\big\lvert f(x)-f(x-y)\big\rvert\mathrm{d}y\\
    &=\int_{\mathbb{R}^N}\lvert h(z)\rvert\sup_{x\in\mathbb{R}^N}\big\lvert f(x)-f(x-z/\Omega)\big\rvert\mathrm{d}z,
\end{align*}
\end{linenomath}
where we have used the substitution \OPtext{$z=\Omega y.$}
Taking limits on both sides of this expression and observing that \[\int_{\mathbb{R}^N}\lvert h(z)\rvert\sup_{x\in\mathbb{R}^N}\big\lvert f(x)-f(x-z/\Omega)\big\rvert\mathrm{d}z
    \leq2\lVert h\rVert_1\sup_{x\in\mathbb{R}^N}\lvert f(x)\rvert
    <\infty,\]
using the Dominated Convergence Theorem, we obtain \[\lim_{\Omega\to\infty}\sup_{x\in\mathbb{R}^N}\big\lvert(f*h_\Omega)(x)-f(x)\big\rvert
    \leq\int_{\mathbb{R}^N}\lvert h(z)\rvert\lim_{\Omega\to\infty}\sup_{x\in\mathbb{R}^N}\big\lvert f(x)-f(x-z/\Omega)\big\rvert\mathrm{d}z.\]
So, it suffices to show that, for all \(z\in\mathbb{R}^N\), \[\lim_{\Omega\to\infty}\sup_{x\in\mathbb{R}^N}\big\lvert f(x)-f(x-z/\Omega)\big\rvert=0.\]
To this end, let \(\varepsilon>0\) and \(z\in\mathbb{R}^N\) be arbitrary.
Since \(f\in C_0(\mathbb{R}^N)\), there exists \(r>0\) sufficiently large such that \(|f(x)|<\varepsilon/2\) for all \(x\in\mathbb{R}^N\setminus\overline{B(0,r)}\), where \(\overline{B(0,r)}\subset\mathbb{R}^N\) is the closed ball of radius \(r\) centered at the origin.
Let \OPtext{$\mathcal{B}:=\overline{B(0,r+\lVert z/\Omega\rVert_2)},$} so that for each \(x\in\mathbb{R}^N\setminus\mathcal{B}\) we have both \(x\) and \OPtext{$x-z/\Omega$} in \(\mathbb{R}^N\setminus\overline{B(0,r)}\).
Thus, both \(|f(x)|<\varepsilon/2\) and \OPtext{$\lvert f(x-z/\Omega)\rvert<\varepsilon/2,$} implying that
\[\sup_{x\in\mathbb{R}^N\setminus\mathcal{B}}\big\lvert f(x)-f(x-z/\Omega)\big\rvert
    <\varepsilon.\]
Hence, we obtain
\begin{linenomath}
\begin{align*}
&\lim_{\Omega\to\infty}\sup_{x\in\mathbb{R}^N}\big\lvert f(x)-f(x-z/\Omega)\big\rvert\\
    &\quad\leq\lim_{\Omega\to\infty}\max\Big\{\sup_{x\in\mathcal{B}}\big\lvert f(x)-f(x-z/\Omega)\big\rvert,\sup_{x\in\mathbb{R}^N\setminus\mathcal{B}}\big\lvert f(x)-f(x-z/\Omega)\big\rvert\Big\}\\
    &\quad\leq\max\Big\{\varepsilon,\lim_{\Omega\to\infty}\sup_{x\in\mathcal{B}}\big\lvert f(x)-f(x-z/\Omega)\big\rvert\Big\}.
\end{align*}
\end{linenomath}
Now, as \(\mathcal{B}\) is a compact subset of \(\mathbb{R}^N\), the continuous function \(f\) is uniformly continuous on \(\mathcal{B}\), and so the remaining limit and supremum may be freely interchanged, whereby continuity of \(f\) yields \[\lim_{\Omega\to\infty}\sup_{x\in\mathcal{B}}\big\lvert f(x)-f(x-z/\Omega)\big\rvert
    =\sup_{x\in\mathcal{B}}\lim_{\Omega\to\infty}\big\lvert f(x)-f(x-z/\Omega)\big\rvert
    =0.\]
Since \(\varepsilon>0\) may be taken arbitrarily small, we have proved the result.
\end{proof}

\begin{remark}\label{rem: hoelder}
While Lemma~\ref{lem: conv equality} does the approximation we aim for, it gives no indication of how fast \[\sup_{x\in\mathbb{R}^N}\big\lvert(f*h_\Omega)(x)-f(x)\big\rvert\]
decays in terms of $\Omega$ or the dimension $N.$ Assuming $h(z)=g(z(1))\cdots g(z(N))$ for some nonnegative $g$ (which is how we will choose $h$ in Section~\ref{sec: step 2}) and $f$ to be $\beta$-H{\"o}lder continuous for some fixed $\beta\in(0,1)$ yields that
\begin{linenomath}
\begin{align*}
    \sup_{x\in\mathbb{R}^N}\big\lvert(f*h_\Omega)(x)-f(x)\big\rvert&\leq\int_{\mathbb{R}^N}\lvert h(z)\rvert\sup_{x\in\mathbb{R}^N}\big\lvert f(x)-f(x-z/\Omega)\big\rvert\mathrm{d}z\\
    &\lesssim\Omega^{-\beta}\int_{\mathbb{R}^N}\lVert z\rVert_2^\beta h(z)\mathrm{d}z\\
    &\leq\Omega^{-\beta}\bigg(\int_{\mathbb{R}^N}\Big(z(1)^2+\cdots+z(N)^2\Big) h(z)\mathrm{d}z\bigg)^{\beta/2}\\
    &\leq\Omega^{-\beta}\bigg(N\max_{j\in\{1,\ldots,N\}}\int_{\mathbb{R}^N}z(j)^2h(z)\mathrm{d}z\bigg)^{\beta/2}\\
    &=(\sqrt{N}/\Omega)^\beta\bigg(\max_{j\in\{1,\ldots,N\}}\int_{\mathbb{R}}z(j)^2g(z(j))\mathrm{d}z(j)\bigg)^{\beta/2}\\
    &=(\sqrt{N}/\Omega)^\beta\bigg(\int_{\mathbb{R}}z(1)^2g(z(1))\mathrm{d}z(1)\bigg)^{\beta/2}\\
    &\lesssim(\sqrt{N}/\Omega)^\beta
\end{align*}
\end{linenomath}
where the third inequality follows from Jensen's inequality. 
\end{remark}

\subsubsection{Proof of \ref{lem: step 2}: 
The limit-integral representation}\label{sec: step 2}

Let $A\in C^\infty(\mathbb{R})$ be \emph{any} even function supported on $[-\tfrac{1}{2},\tfrac{1}{2}]$ s.t.\ $\lVert A\rVert_2=1.$ Then $\phi=A*A\in C^\infty(\mathbb{R})$ is an even function supported on $[-1,1]$ s.t.\ $\phi(0)=1.$ Lemma~\ref{lem: Step 1} implies that
\begin{linenomath}
\begin{align}\label{eq: conv}
    f(x)=\lim_{\Omega\to\infty}(f*h_\Omega)(x)
\end{align}
\end{linenomath}
uniformly in $x\in K$ for \emph{any} $h\in L^1(\mathbb{R}^N)$ satisfying $\int_{\mathbb{R}^N}h(z)\mathrm{d}z=1.$ We choose \[h(z)=\frac{1}{(2\pi)^N}\int_{\mathbb{R}^N}\exp(i\langle w,z\rangle)\prod_{j=1}^N\phi(w(j))\mathrm{d}w\]
which the reader may recognize as the (inverse) Fourier transform of $\prod_{j=1}^N\phi(w(j))$. As we announced in Remark~\ref{rem: hoelder}, $h(z)=g(z(1))\cdots g(z(N)),$ where (using the convolution theorem)
\begin{linenomath}
\begin{align*}
    g(z(j))&=\frac{1}{2\pi}\int_\mathbb{R}\exp(iw(j)z(j))\phi(w(j))\mathrm{d}w(j)\\
    &=\frac{1}{2\pi}\int_{\mathbb{R}}\exp(iw(j)z(j))(A*A)(w(j))\mathrm{d}w(j)\\
    &=2\pi\bigg(\frac{1}{2\pi}\int_{\mathbb{R}}\exp(iw(j)z(j))A(w(j))\mathrm{d}w(j)\bigg)^2\geq0
\end{align*}
\end{linenomath}
Moreover, since $g$ is the Fourier transform of an even function, $h$ is real-valued and also even. In addition, since $\phi$ is smooth, $h$ decays faster than the reciprocal of any polynomial (as follows from repeated integration by parts and the Riemann--Lebesgue lemma), so $h\in L^1(\mathbb{R}^N).$ Thus, Fourier inversion yields \[\int_{\mathbb{R}^N} h(z)\mathrm{d}z=\int_{\mathbb{R}^N}\exp(-i\langle w,z\rangle)h(z)\mathrm{d}z\Big\rvert_{w=0}=\prod_{j=1}^N\phi(0)=1,\]
which justifies our application of Lemma~\ref{lem: Step 1}. Expanding the right-hand side of \eqref{eq: conv} (using the scaling property of the Fourier transform) yields that
\begin{linenomath}
\begin{align}
    (f*h_\Omega)(x)&=\int_{\mathbb{R}^N}f(y)h_\Omega(x-y)\mathrm{d}y=\frac{1}{(2\pi)^N}\int_{K}f(y)\int_{\mathbb{R}^N}\exp(i\langle w,x-y\rangle)\prod_{j=1}^N\phi(w(j)/\Omega)\mathrm{d}w\mathrm{d}y\notag\\
    &=\frac{1}{(2\pi)^N}\int_{K}\int_{[-\Omega,\Omega]^N}f(y)\cos(\langle w,x-y\rangle)\prod_{j=1}^N\phi(w(j)/\Omega)\mathrm{d}w\mathrm{d}y\label{eq: double int}
\end{align}
\end{linenomath}
because $\phi$ is even and supported on $[-1,1].$ Since \eqref{eq: double int} is an iterated integral of a continuous function over a compact set, Fubini's theorem readily applies, yielding \[f(x)=\lim_{\Omega\to\infty}(f*h_\Omega)(x)=\lim_{\Omega\to\infty}\frac{1}{(2\pi)^N}\int_{K\times[-\Omega,\Omega]^N}f(y)\cos(\langle w,x-y\rangle)\prod_{j=1}^N\phi(w(j)/\Omega)\mathrm{d}y\mathrm{d}w.\]
Since $\lvert\langle w,x-y\rangle\rvert\leq\lVert x-y\rVert_1\lVert w\rVert_\infty\leq 2N\mathrm{rad}(K)\Omega\leq(L+\tfrac{1}{2})\pi,$ it follows that
\begin{linenomath}
\begin{align}\label{eqn: cos-int rep}
    f(x)=\lim_{\Omega\to\infty}\frac{1}{(2\pi)^N}\int_{K\times[-\Omega,\Omega]^N}f(y)\cos_\Omega(\langle w,x-y\rangle)\prod_{j=1}^N\phi(w(j)/\Omega)\mathrm{d}y\mathrm{d}w
\end{align}
\end{linenomath}
where $\cos_\Omega$ is defined in \eqref{eqn: trunc cos}.

With the representation~\eqref{eqn: cos-int rep} in hand, we now seek to reintroduce the general activation function~\(\rho\).
To this end, since \(\cos_{\Omega}\in C_c(\mathbb{R})\subset C_0(\mathbb{R})\) we may apply the convolution identity~\eqref{eqn: lim-int rep 2} with \(f\) replaced by \(\cos_{\Omega}\) to obtain \(\cos_{\Omega}(z)=\lim_{\alpha\rightarrow\infty}(\cos_{\Omega}*h_{\alpha})(z)\) uniformly for all \(z\in\mathbb{R}\), where
\OPtext{$h_\alpha(z)=\alpha\rho(\alpha z).$}
Using this representation of \(\cos_{\Omega}\) in~\eqref{eqn: cos-int rep}, it follows that \[f(x)
    =\lim_{\Omega\rightarrow\infty}\frac{1}{(2\pi)^N}
    \int_{K\times[-\Omega,\Omega]^N}
    f(y)\Big(\lim_{\alpha\rightarrow\infty}\big(\cos_{\Omega}*h_{\alpha}\big)\big(\langle w,x-y\rangle\big)\Big)\prod_{j=1}^N\phi(w(j)/\Omega)\mathrm{d}y\mathrm{d}w\]
holds uniformly for all \(x\in K\).
Since \(f\) is continuous and the convolution \(\cos_{\Omega}*h_{\alpha}\) is uniformly continuous and \OPtext{uniformly bounded in $\alpha$ by $\lVert\rho\rVert_1$ (see below),} the fact that the domain \(K\times[-\Omega,\Omega]^N\) is compact then allows us to bring the limit as \(\alpha\) tends to infinity outside the integral in this expression via the Dominated Convergence Theorem, which gives us
\begin{linenomath}
\begin{align}\label{eqn: intermediate lim-int rep}
    f(x)
    =\lim_{\Omega\rightarrow\infty}\lim_{\alpha\rightarrow\infty}\frac{1}{(2\pi)^N}
    \int_{K\times[-\Omega,\Omega]^N}
    f(y)\big(\cos_{\Omega}*h_{\alpha}\big)\big(\langle w,x-y\rangle\big)\prod_{j=1}^N\phi(w(j)/\Omega)\mathrm{d}y\mathrm{d}w
\end{align}
\end{linenomath}
uniformly for every \(x\in K\).
\OPtext{%
The uniform boundedness of the convolution follows from the fact that
\begin{linenomath}
\begin{align}\label{eq: cos conv}
  (\cos_\Omega*h_\alpha)(z)=\int_{\mathbb{R}}\cos_\Omega(z-u)h_\alpha(u)\mathrm{d}u=\int_{\mathbb{R}}\cos_\Omega(z-v/\alpha)\rho(v)\mathrm{d}v,  
\end{align}
\end{linenomath}
where $v=\alpha u.$
}%
\begin{remark}
It should be noted that we are unable to swap the order of the limits in~\eqref{eqn: intermediate lim-int rep}
\OPtext{since} \(\cos_{\Omega}\) is not in \(C_0(\mathbb{R})\) when \(\Omega\) is allowed to be infinite.
\end{remark}

\begin{remark}\label{rem: alpha}
Complementing Remark \ref{rem: hoelder}, we will now elucidate how fast
\begin{linenomath}
\begin{align*}
    \lvert\cos_{\Omega}(z)-(\cos_{\Omega}*h_{\alpha})(z)\rvert
\end{align*}
\end{linenomath}
decays in terms of $\alpha.$ Using the fact that $\int_\mathbb{R}\rho(z)\mathrm{d}z=1,$ \eqref{eq: cos conv} and the triangle inequality allows us to bound the absolute difference above by
\begin{linenomath}
\begin{align*}
    \int_{\mathbb{R}}\lvert\cos_\Omega(z)-\cos_\Omega(z-v/\alpha)\rvert\cdot\lvert\rho(v)\rvert\mathrm{d}v.
\end{align*}
\end{linenomath}
Since $\cos_\Omega$ is 1-Lipschitz, it follows that the above integral is bounded by $\int_\mathbb{R}\lvert v\rho(v)\rvert\,dv/\alpha.$
\end{remark}

To complete this step of the proof, observe that the definition of \(\cos_{\Omega}\) allows us to write
\begin{linenomath}
\begin{align}\label{eqn: unif cos rep}
    (\cos_{\Omega}*h_{\alpha})(z)
    =\alpha\int_{\mathbb{R}}\cos_{\Omega}(u)\rho\big(\alpha(z-u)\big)\mathrm{d}u
    =\alpha\int_{-\frac{\pi}{2}(2L+1)}^{\frac{\pi}{2}(2L+1)}\cos_{\Omega}(u)\rho\big(\alpha(z-u)\big)\mathrm{d}u
\end{align}
\end{linenomath}
By substituting~\eqref{eqn: unif cos rep} into~\eqref{eqn: intermediate lim-int rep}, we then obtain \[f(x)
    =\lim_{\Omega\rightarrow\infty}\lim_{\alpha\rightarrow\infty}
    \frac{\alpha}{(2\pi)^N}
    \int_{K(\Omega)}
    f(y)\cos_{\Omega}(u)\rho\Big(\alpha\big(\langle w,x-y\rangle-u\big)\Big)\prod_{j=1}^N\phi(w(j)/\Omega)\mathrm{d}y\mathrm{d}w\mathrm{d}u\]
uniformly for all \(x\in K\), where \(K(\Omega):=K\times[-\Omega,\Omega]^N\times[-\frac{\pi}{2}(2L+1),\frac{\pi}{2}(2L+1)]\).
In this way, recalling that 
  \OPtext{$  F_{\alpha,\Omega}(y,w,u)
    :=\frac{\alpha}{(2\pi)^N}f(y)\cos_{\Omega}(u)\prod_{j=1}^N\phi(w(j)/\Omega),$} and 
    $b_{\alpha}(y,w,u)
    :=-\alpha(\langle w,y\rangle+u)$
for \(y,w\in\mathbb{R}^N\) and \(u\in\mathbb{R}\), we conclude the proof.
\if{
\begin{lemma}\label{lem: step 2}
Let \(f\in C_c(\mathbb{R}^N)\) and \(\rho\in L^1(\mathbb{R})\) with \(K:=\mathrm{supp}(f)\) and \(\int_{\mathbb{R}}\rho(z)\mathrm{d}z=1\).
Define \(F_{\alpha,\Omega}\) and \(b_{\alpha}\) as in~\eqref{eqn: F-b} for all \(\Omega\in\mathbb{R}^N\) and \(\alpha\in\mathbb{R}\).
Then we have
\begin{align}\label{eqn: final lim-int rep}
    f(x)
    =\lim_{\Omega\rightarrow\infty}\lim_{\alpha\rightarrow\infty}
    \int_{K(\Omega)}
    F_{\alpha,\Omega}(y,w,u)\rho\big(\alpha\langle w,x\rangle+b_{\alpha}(y,w,u)\big)\mathrm{d}y\mathrm{d}w\mathrm{d}u
\end{align}
uniformly for every \(x\in K\), where \(K(\Omega):=K\times[-\Omega,\Omega]^N\times[-\frac{\pi}{2}(2L+1),\frac{\pi}{2}(2L+1)]\) and \(L:=\lceil\frac{2N}{\pi}\mathrm{rad}(K)\Omega-\frac{1}{2}\rceil\).
\end{lemma}}\fi

\subsubsection{Proof of Lemma \ref{lem: part 3}: Monte-Carlo integral approximation}\label{sec: step 3}

The next step in the proof of Theorem~\ref{thm: IP95} is to approximate the integral in~\eqref{eqn: final lim-int rep} using the Monte-Carlo method.
To this end, let \(\{y_k\}_{k=1}^n\), \(\{w_k\}_{k=1}^n\), and \(\{u_k\}_{k=1}^n\) be independent samples drawn uniformly from \(K\), \([-\Omega,\Omega]^N\), and \([-\frac{\pi}{2}(2L+1),\frac{\pi}{2}(2L+1)]\), respectively, and consider the sequence of random variables \(\{I_n(x)\}_{n=1}^{\infty}\) defined by
\begin{linenomath}
\begin{align}\label{eqn: MC sum}
    I_n(x)
    :=\frac{\mathrm{vol}(K(\Omega))}{n}\sum_{k=1}^n
    F_{\alpha,\Omega}(y_k,w_k,u_k)\rho\big(\alpha\langle w_k,x\rangle+b_{\alpha}(y_k,w_k,u_k)\big)
\end{align}
\end{linenomath}
for each \(x\in K\), where we note that \(\mathrm{vol}(K(\Omega))=(2\Omega)^N\pi(2L+1)\mathrm{vol}(K)\).
If we also define
\begin{linenomath}
\begin{align}\label{eqn: MC int appendix}
    I(x;p)
    :=\int_{K(\Omega)}
    \Big(F_{\alpha,\Omega}(y,w,u)\rho\big(\alpha\langle w,x\rangle+b_{\alpha}(y,w,u)\big)\Big)^p\mathrm{d}y\mathrm{d}w\mathrm{d}u
\end{align}
\end{linenomath}
for \(x\in K\) and \(p\in\mathbb{N}\), then we want to show that
\begin{linenomath}
\begin{align}\label{eqn: MC approx}
    \mathbb{E}\int_K|I(x;1)-I_n(x)|^2\mathrm{d}x
    =O(1/n)
\end{align}
\end{linenomath}
\OPtext{as $n\to\infty,$} where the expectation is taken with respect to the joint distribution of the random samples \(\{y_k\}_{k=1}^n\), \(\{w_k\}_{k=1}^n\), and \(\{u_k\}_{k=1}^n\).
For this, it suffices to find a constant \(C(f,\rho,\alpha,\Omega,N)<\infty\) independent of \(n\) satisfying \[\int_K\mathbb{E}|I(x;1)-I_n(x)|^2\mathrm{d}x
    \leq\frac{C(f,\rho,\alpha,\Omega,N)}{n}.\]
Indeed, an application of Fubini's theorem would then yield \[\mathbb{E}\int_K
    |I(x;1)-I_n(x)|^2\mathrm{d}x
    \leq\frac{C(f,\rho,\alpha,\Omega,N)}{n},\]
which implies~\eqref{eqn: MC approx}.
To determine such a constant, we first observe by  Theorem~\ref{thm: MC MSE} that \[\mathbb{E}|I(x;1)-I_n(x)|^2
    =\frac{\mathrm{vol}^2(K(\Omega))\sigma(x)^2}{n},\]
where we define the variance term \[\sigma(x)^2
    :=\frac{I(x;2)}{\mathrm{vol}(K(\Omega))}-\frac{I(x;1)^2}{\mathrm{vol}^2(K(\Omega))}\]
for \(x\in K\).
\OPtext{Since $\lVert\phi\rVert_\infty=1$ (see Lemma~\ref{lem: phi} below), it follows that}\[
    \lvert F_{\alpha,\Omega}(y,w,u)\rvert
    =\frac{\alpha}{(2\pi)^N}\lvert f(y)\rvert\cdot\lvert\cos_{\Omega}(u)\rvert\prod_{j=1}^N\lvert\phi(w(j)/\Omega)\rvert
    \leq\frac{\alpha M}{(2\pi)^N}\]
for all \(y,w\in\mathbb{R}^N\) and \(u\in\mathbb{R}\), where \(M:=\sup_{x\in K}|f(x)|<\infty\), we obtain the following simple bound on the variance term
\begin{linenomath}
\begin{align}\label{eqn: sigma bound}
    \sigma(x)^2
    \leq\frac{I(x;2)}{\mathrm{vol}(K(\Omega))}
    \leq\frac{\alpha^2M^2}{(2\pi)^{2N}\mathrm{vol}(K(\Omega))}\int_{K(\Omega)}\Big|\rho\big(\alpha\langle w,x\rangle+b_{\alpha}(y,w,u)\big)\Big|^2\mathrm{d}y\mathrm{d}w\mathrm{d}u.
\end{align}
\end{linenomath}
Since we assume \OPtext{$\rho\in L^\infty(\mathbb{R}),$} we then have
\begin{linenomath}
\begin{align*}
    \int_K\mathbb{E}|I(x;1)-I_n(x)|^2\mathrm{d}x
    &=\frac{\mathrm{vol}^2(K(\Omega))}{n}\int_K\sigma(x)^2\mathrm{d}x\\
    &\leq\frac{\alpha^2M^2\mathrm{vol}(K(\Omega))}{(2\pi)^{2N}n}\int_{K\times K(\Omega)}\Big|\rho\big(\alpha\langle w,x\rangle+b_{\alpha}(y,w,u)\big)\Big|^2\mathrm{d}x\mathrm{d}y\mathrm{d}w\mathrm{d}u\\
    &=\frac{\alpha^2M^2\mathrm{vol}^2(K(\Omega))\mathrm{vol}(K)\lVert\rho\rVert_\infty^2}{(2\pi)^{2N}n}.
\end{align*}
\end{linenomath}
Substituting the value of \(\mathrm{vol}(K(\Omega))\), we obtain \[C(f,\rho,\alpha,\Omega,N)
    :=\alpha^2M^2(\Omega/\pi)^{2N}\pi^2(2L+1)^2\mathrm{vol}^3(K)
    \Vert\rho\Vert_\infty^2\]
is a suitable choice for the desired constant.

Now that we have established~\eqref{eqn: MC approx}, we may rewrite the random variables \(I_n(x)\) in a more convenient form.
To this end, we change the domain of the random samples \(\{w_k\}_{k=1}^n\) to \([-\alpha\Omega,\alpha\Omega]^N\) and define the new random variables \(\{b_k\}_{k=1}^n\subset\mathbb{R}\) by \(b_k:=-(\langle w_k,y_k\rangle+\alpha u_k)\) for each \(k=1,\ldots,n\).
In this way, if we denote \[v_k
    :=\frac{\mathrm{vol}(K(\Omega))}{n}F_{\alpha,\Omega}\Big(y_k,\frac{w_k}{\alpha},u_k\Big)\]
for each \(k=1,\ldots,n\), the random variables \(\{f_n\}_{n=1}^{\infty}\) defined by 

\[f_n(x)
    :=\sum_{k=1}^nv_k\rho\big(\langle w_k,x\rangle+b_k\big)\]
satisfy \(f_n(x)=I_n(x)\) for every \(x\in K\).
Combining this with~\eqref{eqn: MC approx}, we have proved Lemma~\ref{lem: part 3}.

\begin{lemma}\label{lem: phi}
$\lVert\phi\rVert_\infty=1.$
\end{lemma}
\begin{proof}
It suffices to prove that $\lvert\phi(z)\rvert\leq1$ for all $z\in\mathbb{R}$ because $\phi(0)=1.$ By Cauchy--Schwarz,
\begin{linenomath}
\begin{align*}
    \lvert\phi(z)\rvert&=\bigg\lvert\int_{\mathbb{R}}A(u)A(z-u)\mathrm{d}u\bigg\rvert\leq\sqrt{\int_{\mathbb{R}}\displaystyle A(u)A(u)\mathrm{d}u\int_{\mathbb{R}}\displaystyle A(z-u)A(z-u)\mathrm{d}u}\\
    &=\sqrt{\int_{\mathbb{R}}\displaystyle A(u)A(0-u)\mathrm{d}u\int_{\mathbb{R}}\displaystyle A(v)A(-v)\mathrm{d}v}=\sqrt{\phi(0)\phi(0)}=1
\end{align*}
\end{linenomath}
because $A$ is even.
\end{proof}

\subsubsection{Proof of Theorem~\ref{thm: IP95_short} when \OPtext{$\rho'\in L^1(\mathbb{R})\cap L^\infty(\mathbb{R})$}}
\label{sec: rho_prime_proof}
Let \(f\in C_c(\mathbb{R}^N)\) with \(K:=\mathrm{supp}(f)\) and suppose \(\varepsilon>0\) is fixed.
Take the activation function \(\rho\colon\mathbb{R}\rightarrow\mathbb{R}\) to be differentiable with \OPtext{$\rho^{\prime}\in L^1(\mathbb{R})\cap L^\infty(\mathbb{R}).$}
We wish to show that there exists a sequence of RVFL networks \(\{f_n\}_{n=1}^{\infty}\) defined on \(K\) which satisfy the asymptotic error bound \[\mathbb{E}\int_K|f(x)-f_n(x)|^2\mathrm{d}x\leq\varepsilon+O(1/n)\]
\OPtext{as $n\to\infty.$} The proof of this result is a minor modification of \OPtext{second step} in the proof of Theorem~\ref{thm: IP95}.

\OPtext{If we redefine $h_\alpha(z)$ as $\alpha\rho'(\alpha z),$ then~\eqref{eqn: intermediate lim-int rep} plainly still holds and~\eqref{eqn: unif cos rep} reads}
\[(\cos_{\Omega}*h_{\alpha})(z)
    =\alpha\int_{\mathbb{R}}\cos_{\Omega}(u)\rho^{\prime}\big(\alpha(z-u)\big)\mathrm{d}u.
\]
\OPtext{R}ecalling the definition of \(\cos_{\Omega}\) in~\eqref{eqn: trunc cos} and integrating by parts, we obtain
\begin{linenomath}
\begin{align*}
    (\cos_{\Omega}*h_{\alpha})(z)
    &=\alpha\int_{\mathbb{R}}\cos_{\Omega}(u)\rho^{\prime}\big(\alpha(z-u)\big)\mathrm{d}u\\
    &=-\int_{-\frac{\pi}{2}(2L+1)}^{\frac{\pi}{2}(2L+1)}\cos_{\Omega}(u)d\rho(\alpha(z-u))\\
    &=-\cos_{\Omega}(u)\rho(\alpha(z-u))\Big|_{-\frac{\pi}{2}(2L+1)}^{\frac{\pi}{2}(2L+1)}+\int_{-\frac{\pi}{2}(2L+1)}^{\frac{\pi}{2}(2L+1)}\rho(\alpha(z-u))d\cos_{\Omega}(u)\\
    &=-\int_{\mathbb{R}}\sin_{\Omega}(u)\rho\big(\alpha(z-u)\big)\mathrm{d}u
\end{align*}
\end{linenomath}
for all \(z\in\mathbb{R}\), where \(L:=\lceil\frac{2N}{\pi}\mathrm{rad}(K)\Omega-\frac{1}{2}\rceil\) and \(\sin_{\Omega}\colon\mathbb{R}\rightarrow[-1,1]\) is defined analogously to~\eqref{eqn: trunc cos}.
Substituting this representation of \((\cos_{\Omega}*h_{\alpha})(z)\) into~\eqref{eqn: intermediate lim-int rep} then yields 

\[f(x) =\lim_{\Omega\rightarrow\infty}\lim_{\alpha\rightarrow\infty}
    \frac{-\alpha}{(2\pi)^N}
    \int_{K(\Omega)}
    f(y)\sin_{\Omega}(u)\rho\big(\alpha(\langle w,x-y\rangle-u)\big)\prod_{j=1}^N\phi(w(j)/\Omega)\mathrm{d}y\mathrm{d}w\mathrm{d}u\]
uniformly for every \(x\in K\).
Thus, if we replace the definition of \(F_{\alpha,\Omega}\) in~\eqref{eqn: F-b} by \[F_{\alpha,\Omega}(y,w,u)
    :=\frac{-\alpha}{(2\pi)^N}f(y)\sin_{\Omega}(u)\prod_{j=1}^N\phi(w(j)/\Omega)\]
for \(y,w\in\mathbb{R}^N\) and \(u\in\mathbb{R}\), we again obtain the uniform representation~\eqref{eqn: final lim-int rep} for all \(x\in K\).
The remainder of the proof proceeds from this point exactly as in the proof of Theorem~\ref{thm: IP95}.

\subsubsection{Proof of Theorem \ref{thm: probabilistic manifold Igelnik-Pao}}\label{sec: Manifold proof}
We wish to show that there exist sequences of RVFL networks \(\{\tilde{f}_{n_j}\}_{n_j=1}^{\infty}\) defined on \(\phi_j(U_j)\) for each \(j\in J\) which together satisfy the error bound \[\int_{\mathcal{M}}\bigg|f(x)\quad-\sum_{\{j\in J\colon x\in U_j\}}(\tilde{f}_{n_j}\circ\phi_j)(x)\bigg|^2\mathrm{d}x<\varepsilon
\]
with probability at least \(1-\eta\) for \(\{n_j\}_{j\in J}\) sufficiently large.
The proof is obtained by showing that
\begin{linenomath}
\begin{align}\label{eqn: manifold ptwise bound}
    \bigg|f(x)\quad-\sum_{\{j\in J\colon x\in U_j\}}(\tilde{f}_{n_j}\circ\phi_j)(x)\bigg|
    <\sqrt{\frac{\varepsilon}{\mathrm{vol}(\mathcal{M})}}
\end{align}
\end{linenomath}
holds uniformly for \(x\in\mathcal{M}\) with high probability.

We begin as in the proof of Theorem~\ref{thm: manifold Igelnik-Pao} by applying the representation~\eqref{eqn: manifold funct rep} for \(f\) on each chart \((U_j,\phi_j)\), which gives us
\begin{linenomath}
\begin{align}\label{eqn: prob man bound 1}
    \bigg|f(x)\quad-\sum_{\{j\in J\colon x\in U_j\}}(\tilde{f}_{n_j}\circ\phi_j)(x)\bigg|
    &\leq\sum_{\{j\in J\colon x\in U_j\}}\Big|(\hat{f}_j\circ\phi_j)(x)-(\tilde{f}_{n_j}\circ\phi_j)(x)\Big|
\end{align}
\end{linenomath}
for all \(x\in\mathcal{M}\).
Now, since we have already seen that \(\hat{f}_j\in C_c(\mathbb{R}^d)\) for each \(j\in J\), Theorem~\ref{thm: probabilistic Igelnik-Pao} implies that for any \(\varepsilon_j>0\), there exist constants \(\alpha_j,\Omega_j>0\) and hidden-to-output layer weights \(\{v_k^{(j)}\}_{k=1}^{n_j}\subset\mathbb{R}\) for each \(j\in J\) such that for any
\begin{linenomath}
\begin{align}\label{eqn: manifold net bound}
    \delta_j
    <\frac{\sqrt{\varepsilon_j}}{8\sqrt{2d}\kappa\alpha_j^2M_j\Omega_j(\Omega_j/\pi)^d\mathrm{vol}^{3/2}(\phi_j(U_j))(\pi+2d\mathrm{rad}(\phi_j(U_j))\Omega)}
\end{align}
\end{linenomath}
we have \[\Big|\hat{f}_j(z)-\tilde{f}_{n_j}(z)\Big|
    <\sqrt{\frac{\varepsilon_j}{2\mathrm{vol}(\phi_j(U_j))}}\]
uniformly for all \(z\in\phi_j(U_j)\) with probability at least \(1-\eta_j\), provided the number of nodes \(n_j\) satisfies
\begin{linenomath}
\begin{align}\label{eqn: manifold node bound}
    n_j
    &\geq\frac{c\Sigma^{(j)}\alpha_j(\Omega_j/\pi)^d(\pi+2d\mathrm{rad}(\phi_j(U_j))\Omega_j)\log(3\eta_j^{-1}\mathcal{N}(\delta_j,\phi_j(U_j)))}{\sqrt{\varepsilon_j}\log\big(1+\frac{\sqrt{\varepsilon_j}}{\Sigma^{(j)} \alpha_j(\Omega_j/\pi)^d(\pi+2d\mathrm{rad}(\phi_j(U_j))\Omega_j)}\big)},
\end{align}
\end{linenomath}
where \(c>0\) is a numerical constant and \OPtext{$\Sigma^{(j)}:=2C^{(j)}\sqrt{2\mathrm{vol}(\phi_j(U_j))}.$}
Indeed, it suffices to choose 
\[v_k^{(j)}:=\frac{\mathrm{vol}(K(\Omega_j))}{n_j}F_{\alpha_j,\Omega_j}\Big(y_k^{(j)},\frac{w_k^{(j)}}{\alpha_j},u_k^{(j)}\Big)
\]
for each \(k=1,\ldots,n_j\), where
\[K(\Omega_j):=\phi_j(U_j)\times[-\alpha_j\Omega_j,\alpha_j\Omega_j]^d\times[-\tfrac{\pi}{2}(2L_j+1),\tfrac{\pi}{2}(2L_j+1)]
\]
for each \(j\in J\).
Combined with~\eqref{eqn: prob man bound 1}, choosing \(\delta_j\) and \(n_j\) satifying~\eqref{eqn: manifold net bound} and~\eqref{eqn: manifold node bound}, respectively, then yields
\[\bigg|f(x)\quad-\sum_{\{j\in J\colon x\in U_j\}}(\tilde{f}_{n_j}\circ\phi_j)(x)\bigg|
    <\sum_{\{j\in J\colon x\in U_j\}}\sqrt{\frac{\varepsilon_j}{2\mathrm{vol}(\phi_j(U_j))}}
    \leq\sum_{j\in J}\sqrt{\frac{\varepsilon_j}{2\mathrm{vol}(\phi_j(U_j))}}
\]
for all \(x\in\mathcal{M}\) with probability at least \(1-\sum_{\{j\in J\colon x\in U_j\}}\eta_j\geq1-\sum_{j\in J}\eta_j\).
Since we require that~\eqref{eqn: manifold ptwise bound} holds for all \(x\in\mathcal{M}\) with probability at least \(1-\eta\), the proof is then completed by choosing \(\{\varepsilon_j\}_{j\in J}\) and \(\{\eta_j\}_{j\in J}\) such that
\[\varepsilon
    =\frac{\mathrm{vol}(\mathcal{M})}{2}\Big(\sum_{j\in J}\sqrt{\frac{\varepsilon_j}{\mathrm{vol}(\phi_j(U_j))}}\Big)^2
    \quad\text{ and }\quad
    \eta
    =\sum_{j\in J}\eta_j.
\]
In particular, it suffices to choose
\[\varepsilon_j
    =\frac{2\mathrm{vol}(\phi_j(U_j))\,\varepsilon}{|J|^2\mathrm{vol}(\mathcal{M})}
\]
and \(\eta_j=\eta/|J|\) for each \(j\in J\), so that~\eqref{eqn: manifold net bound} and~\eqref{eqn: manifold node bound} become
\begin{linenomath}
\begin{align*}
    \delta_j
    &<\frac{\sqrt{\varepsilon}}{8\lvert J\rvert\sqrt{d\mathrm{vol}(\mathcal{M})}\kappa\alpha_j^2M_j\Omega_j(\Omega_j/\pi)^d\mathrm{vol}(\phi_j(U_j))(\pi+2d\mathrm{rad}(\phi_j(U_j))\Omega)},\\
    n_j
    &\geq\frac{2c\lvert J\rvert\sqrt{\mathrm{vol}(\mathcal{M})}C^{(j)}\alpha_j(\Omega_j/\pi)^d(\pi+2d\mathrm{rad}(\phi_j(U_j))\Omega_j)\log(3\lvert J\rvert\eta^{-1}\mathcal{N}(\delta_j,\phi_j(U_j)))}{\sqrt{\varepsilon}\log\big(1+\frac{\sqrt{\varepsilon}}{2\lvert J\rvert\sqrt{\mathrm{vol}(\mathcal{M})}C^{(j)} \alpha_j(\Omega_j/\pi)^d(\pi+2d\mathrm{rad}(\phi_j(U_j))\Omega_j)}\big)},
\end{align*}
\end{linenomath}
as desired.

\section{Discussion}\label{sec: discussion}

The central topic of this paper is the study of the approximation properties of a randomized variation of shallow neural networks known as RVFL. In contrast with the classical single-layer neural networks, training of an RVFL involves only learning the output weights, while the input weights and biases of all the nodes are selected at random from an appropriate distribution and stay fixed throughout the training. The main motivation for studying the properties of such networks is as follows:

\begin{enumerate}
    \item Random weights are often  utilized as an initialization for a NN training procedure. Thus, establishing the properties of the RVFL networks is an important first step toward understanding  how random weights are transformed during training.

    \item Due to their much more computationally efficient training process, the RVFL networks proved to be a valuable alternative to the classical SLFNs. They were successfully used in several modern applications, especially those that require frequent re-training of a neural network~\citep{chen1999rapid, tang2018noniterative,dash2018indian}.
\end{enumerate}

Despite their practical and theoretical importance, results providing rigorous mathematical analysis of the properties of RVFLs are rare. The work of Igelnik and Pao \cite{igelnik1995stochastic} showed that RVFL networks are  universal approximators for the class of continuous, compactly supported functions and established the asymptotic convergence rate of the expected approximation error as a function of the number of nodes in the hidden layer. While this result served as a theoretical justification for using RVFL networks in practice, a close examination led us to the conclusion that the proofs in \cite{igelnik1995stochastic} contained several technical errors. 

In this paper, we offer a revision and a modification of the proof methods from \cite{igelnik1995stochastic} that allow us to prove a corrected, slightly weaker version of the result announced by Igelnik and Pao. We further build upon their work and show a non-asymptotic probabilistic (instead of on average) approximation result, which gives an explicit bound on the number of hidden layer nodes that are required to achieve the desired approximation accuracy with the desired level of certainty (that is, with high enough probability). In addition to that, we extend the obtained result to the case when the function is supported on a compact, low-dimensional submanifold of the ambient space.

While our work closes some of the gaps in the study of the approximation properties of RVFL, we believe that it just starts the discussion and opens many  directions for further research. We briefly outline some of them here.

In our results, the dependence of the required number $n$ of the nodes in the hidden layer on the dimension $N$ of the domain is superexponential, which is likely an artifact of the proof methods we use. We believe this dependence can be improved to be exponential by using a different, more refined approach to the construction of the limit-integral representation of a function. A related interesting direction for future research is to study how the bound on $n$ changes for more restricted classes of (e.g., smooth) functions. 

Another important direction that we did not discuss in this paper is learning the output weights and studying the robustness of the RVFL approximation to the noise in the training data. Obtaining provable robustness guarantees for an RVFL training procedure would be a step towards the robustness analysis of neural networks.

\section*{Conflict of Interest Statement}

The authors declare that the research was conducted in the absence of any commercial or financial relationships that could be construed as a potential conflict of interest.

%


\section*{Acknowledgments}
Deanna Needell was partially supported by NSF  DMS 2108479 and NSF DMS 2011140. Rayan Saab was partially supported by a UCSD senate research award and a Simons fellowship. Palina Salanevich was partially supported by NSF Division of Mathematical Sciences award \#1909457.
The authors thank F. Krahmer, S. Krause-Solberg, and J. Maly for sharing their GMRA code, which they adapted from that provided by M. Maggioni.


\section*{Data Availability Statement}
The code used to obtain the numerical results is available upon direct request sent to the corresponding author.

\bibliographystyle{plain} 
\bibliography{RVFL_frontiers}

\end{document}